\documentclass[twocolumn]{IEEEtran}

\usepackage[T1]{fontenc}
\usepackage[utf8]{inputenc}
\usepackage[english]{babel}
\usepackage[margin=1in]{geometry}
\usepackage{authblk}
\usepackage{macros}  
\newtheorem{definition}{Definition}

\tikzstyle{block} = [draw, fill=white, rectangle,
    minimum height=3em, minimum width=6em]
\tikzstyle{sum} = [draw, fill=white, circle, node distance=1cm]
\tikzstyle{input} = [coordinate]
\tikzstyle{output} = [coordinate]
\tikzstyle{pinstyle} = [pin edge={to-,thin,black}]
\tikzset{roads/.style={line width=0.2cm}}

\usepackage{times}
\usepackage{multirow}

\begin{document}

\title{Non-asymptotic bounds for stochastic optimization with biased noisy gradient oracles}

\author{Nirav Bhavsar$^\dagger$ and  Prashanth L.A.$^\sharp$\thanks{
		$^\dagger$ Department of Computer Science and Engineering,
		Indian Institute of Technology Madras, Chennai,
		E-Mail: niravnb@cse.iitm.ac.in,
		
		$^\sharp$ Department of Computer Science and Engineering,
		Indian Institute of Technology Madras, Chennai,
		E-Mail: prashla@cse.iitm.ac.in.}}



\maketitle

\begin{abstract}
	We introduce biased gradient oracles to capture a setting where the function measurements have an estimation error that can be controlled through a batch size parameter. Our proposed oracles are appealing in several practical contexts, for instance, risk measure estimation from a batch of independent and identically distributed (i.i.d.) samples, or simulation optimization, where the function measurements are `biased' due to computational constraints. In either case, increasing the batch size reduces the estimation error. We highlight the applicability of our biased gradient oracles in a risk-sensitive reinforcement learning setting.  In the stochastic non-convex optimization context, we analyze a variant of the randomized stochastic gradient (RSG) algorithm with a biased gradient oracle. We quantify the convergence rate of this algorithm by deriving non-asymptotic bounds on its performance.  Next, in the stochastic convex optimization setting, we derive non-asymptotic bounds for the last iterate of a stochastic gradient descent (SGD) algorithm with a biased gradient oracle. 
\end{abstract}

\begin{IEEEkeywords}
	Biased gradient oracle, zeroth-order stochastic optimization, simultaneous perturbation, Gaussian smoothing, non-asymptotic bounds.
\end{IEEEkeywords}
\section{Introduction}
\label{sec:intro}

We consider the problem of minimizing a smooth objective function, when the optimization algorithm is provided with biased function  measurements. 
This setting is motivated by practical applications, where the objective function is estimated from a batch dataset, and the estimation scheme is biased. As an example, consider the problem of estimating  Conditional Value-at-Risk (CVaR), a popular risk measure in financial applications, from a batch of independent and identically distributed (i.i.d.) samples. The classic CVaR estimator \cite{rocka} requires estimation of a certain quantile of the underlying distribution, and hence, the resulting estimate is biased. 
As another example, one could consider a simulation optimization problem \cite{fu2015handbook}, where the function measurements are `biased' due to computational constraints. In both examples, increasing the batch size used for estimation decreases the estimation error. 

The `biased stochastic optimization' setting outlined above is more general than the canonical `zeroth-order stochastic optimization` setting because the former features an estimation error that has a positive mean, while the latter usually features an estimation error that vanishes in expectation. We extend the theory of zeroth-order stochastic optimization to our setting by formalizing two oracle models that encapsulate a biased stochastic optimization problem. In each oracle model, an algorithm obtains a noisy and biased estimate of the gradient at any chosen point.  Both oracles feature a batch size parameter that can be used to control an additive estimation error component in the gradient estimates. 
The difference between the two proposed biased gradient oracles is that the first oracle features a bias-variance tradeoff for the gradient estimates, while the second one does not have such a tradeoff.


The biased gradient oracles can be implemented using the simultaneous perturbation \cite{bhatnagar-book,nesterov2011random}  class of algorithms that can provide biased gradient information, using only noisy function measurements. Such an approach can be extended to cover the case of biased function measurements that we consider in this paper. The gradient estimate resulting from a simultaneous perturbation method usually has a bias-variance tradeoff, i.e., the estimate has an additive bias of $\mathcal{ O}(\eta^2)$, where $\eta$ is a parameter to be chosen by the optimization algorithm. The variance of the gradient estimate is $\mathcal{ O}(1/\eta^2)$, and the choice of $\eta$ relates to bias-variance tradeoff \cite{spall1992multivariate,bhatnagar-book,spall2005introduction,hu2016bandit}. Under additional assumptions, one can eschew the bias-variance tradeoff, i.e., reduce the bias without adversely affecting the variance  \cite{nesterov2011random,ghadimi2013stochastic,balasubramanian2018zeroth}.

\begin{table*}
	\centering
	\captionsetup[subtable]{position = below}
	\captionsetup[table]{position=top}
	\caption{Summary of the iteration and sample complexity we obtain for the RSG-BGO algorithm \ref{alg:rsg} and SGD-BGO  algorithm \ref{alg:net} under two oracles for finding an $\epsilon$-stationary or $\epsilon$-optimal point (see Definition \ref{def:esolution}). Here \ref{as:biased_sp_estimation_error} is an oracle that returns a biased gradient estimate with a parameter that controls bias-variance tradeoff, while \ref{as:biased_gs_estimation_error} is an oracle variant where the bias of the gradient estimate can be reduced without adversely affecting the variance. Both oracles have a batch size parameter that controls the estimation error (See Section \ref{sec:oracles} for precise definitions).}
	\begin{subtable}{0.5\textwidth}
		\centering
			\hspace*{-2em}
		\begin{tabular}{|c|c|c|c|}
			\hline
			\multirow{2}{*} {\textbf{Oracle}} &\textbf{Iteration} & \textbf{Sample} & \multirow{2}{*}{\textbf{Reference}} \\ 
			& \textbf{complexity} & \textbf{complexity} &  \\ \hline
 			RSG with \ref{as:biased_sp_estimation_error} & $  \mathcal{O} \left(\frac{1}{\epsilon^3} \right)$ & $ N^2 $ & Theorem \ref{thm:biased_sp_esterr} \\ \hline
			RSG with \ref{as:biased_gs_estimation_error} & $  \mathcal{O} \left(\frac{1}{\epsilon^2} \right) $ & $ N^3 $ & Theorem \ref{thm:biased_gs_esterr} \\ \hline
			RSG with  an &   \multirow{2}{*} { $ \mathcal{O} \left(\frac{1}{\epsilon^2} \right)$ }  &\multirow{2}{*} { $ N $} & \multirow{2}{*} { \cite{ghadimi2013stochastic}} \\ 
			 unbiased gradient  & &  & \\ 
             oracle & &  & \\ \hline
		\end{tabular}
		\caption{Stochastic Non-convex Optimization}
		\label{tab:snco}
	\end{subtable}%
	\hspace*{-1em}
	\begin{subtable}{0.5\textwidth}
		\centering
			\begin{tabular}{|c|c|c|c|}
			\hline
			\multirow{2}{*} {\textbf{Oracle}} &\textbf{Iteration} & \textbf{Sample} & \multirow{2}{*}{\textbf{Reference}} \\ 
			& \textbf{complexity} & \textbf{complexity} &  \\ \hline
			RSG with \ref{as:biased_sp_estimation_error} & $  \mathcal{O} \left(\frac{1}{\epsilon^3} \right)$ & $ N^2 $ & Theorem \ref{thm:biased_sp_esterr_convex} \\ \hline
			RSG with \ref{as:biased_gs_estimation_error} & $  \mathcal{O} \left(\frac{1}{\epsilon^2} \right) $ & $ N^3 $ & Theorem \ref{thm:biased_gs_esterr_convex}\\ \hline
			SGD with \ref{as:biased_sp_estimation_error} & $  \mathcal{O} \left(\frac{1}{\epsilon^3} \right) $ & $ N^2 \log_2 N  $ & Theorem \ref{thm:convex_sp_esterr} \\ \hline
			SGD with \ref{as:biased_gs_estimation_error} & $  \mathcal{O} \left(\frac{1}{\epsilon^2} \right) $ & $ 4N^4 (N^2 - 1) / 3 $ & Theorem \ref{thm:convex_gs_esterr} \\ \hline
			SGD with & \multirow{2}{*} { $  \mathcal{O} \left(\frac{1}{\epsilon^2} \right) $} & \multirow{2}{*} { $ N $ } & \multirow{2}{*} { \cite{jain2019making}}\\ 
			unbiased oracle & & & \\ \hline
		\end{tabular}
		\caption{Stochastic Convex Optimization}
		\label{tab:sco}
	\end{subtable}
	\label{tab:summary}
	\vspace{-3mm}
\end{table*}
The focus of this paper is to understand the rate of convergence of gradient-based methods with inputs from a biased gradient oracle. 
We derive non-asymptotic bounds on the iteration complexity of gradient-based methods for a non-convex as well as a convex objective.  In either case, we derive bounds for gradient-based methods with inputs from the following oracle models: (i) an oracle whose gradient estimates have a parameter for trading off bias against the variance. We shall refer to this oracle as \ref{as:biased_sp_estimation_error}; and (ii) an oracle where the gradient estimates have no bias-variance tradeoff. We shall refer to this oracle as \ref{as:biased_gs_estimation_error} below. Note that both oracles have a batch size parameter for controlling the estimation error.
Table \ref{tab:summary} summarizes our bounds in the convex as well non-convex regimes, under two oracle models.

We now summarize our contributions in the case when the objective is non-convex.
We study the non-asymptotic performance of the randomized stochastic gradient (RSG) algorithm, proposed in \cite{ghadimi2013stochastic}. The case of unbiased gradient information is addressed in the aforementioned reference, and we focus on the cases when RSG is provided inputs from \ref{as:biased_sp_estimation_error} or \ref{as:biased_gs_estimation_error}. 
From our analysis, we  observe that RSG has a sample complexity bound of $\mathcal{O} (1/\epsilon^{3})$ with \ref{as:biased_sp_estimation_error}, and $\mathcal{O} (1/\epsilon^{2})$ with \ref{as:biased_gs_estimation_error}.  
This is not surprising, as \ref{as:biased_sp_estimation_error} provides a gradient estimate whose variance scales inversely with the perturbation constant $\eta$, and this is unlike the estimate from \ref{as:biased_gs_estimation_error}, where such an inverse scaling is absent. 
Our result, when specialized to a setting without the estimation error, matches the corresponding result in  \cite{ghadimi2013stochastic}. 
An advantage with our approach is that, unlike \cite{ghadimi2013stochastic}, we do not require knowledge of the function value at the optima for choosing the perturbation constant $\eta$. 
We demonstrate the applicability of our biased gradient oracle by  considering a risk-sensitive optimization problem in a reinforcement learning setting. We propose a risk-sensitive policy gradient (Risk-PG) algorithm, and show that the analysis of the RSG algorithm with \ref{as:biased_sp_estimation_error} applies to Risk-PG algorithm, while the results under \ref{as:biased_gs_estimation_error} would apply under additional assumptions.

Next, we summarize our contributions in the case when the objective is convex. Using a proof technique that is similar to the one employed in the non-convex case, we provide non-asymptotic bounds for the RSG algorithm with inputs from either \ref{as:biased_sp_estimation_error} or \ref{as:biased_gs_estimation_error}. A disadvantage with the RSG algorithm is that it requires knowledge of the smoothness parameter for choosing the step-size parameter in the gradient descent update iteration. We overcome this dependence by employing a different algorithm that is based on the stochastic gradient descent scheme analyzed in \cite{jain2019making}. We provide non-asymptotic bounds that hold in expectation for the final iterate of the stochastic gradient algorithm with inputs from either \ref{as:biased_sp_estimation_error} or \ref{as:biased_gs_estimation_error}. 
For the case of unbiased gradient information, the authors in  \cite{jain2019making} provide a sample complexity bound of the order $\O\left(1/\epsilon^{2}\right)$. We also provide a similar order bound, when the gradients are obtained from \ref{as:biased_gs_estimation_error}. On the other hand, when gradient estimates from \ref{as:biased_sp_estimation_error} are employed, the bound we obtain is of the order $\O\left(1/\epsilon^{3}\right)$. The latter bound is not surprising, considering a matching minimax lower bound shown in \cite{hu2016bandit}.

\textit{Related work.}
Biased stochastic optimization has been considered before in \cite{hu2016bandit,devolder2014first,Baes09,dAsp08,nguyen2021inexact,bollapragada2019exact,pasupathy2018sampling}.
In \cite{devolder2014first,Baes09,dAsp08}, the authors consider an oracle that outputs biased gradient measurements without any noise component. In contrast, we consider noisy gradient measurements with a controllable estimation error.  
In \cite{hu2016bandit}, which is a closely related work, the authors formalize a biased noisy gradient oracle. They derive an upper bound for a mirror descent scheme, and a minimax lower bound, both for the case of a convex objective. Unlike \cite{hu2016bandit}, our oracle model features an additional estimation error component that is not zero mean.  Our results, when specialized to the case without an estimation error matches the upper bound derived in \cite{hu2016bandit}. Our bounds are for a regular stochastic gradient descent algorithm, with the added advantage that the stepsize we employ does not require knowledge of the underlying smoothness parameter. More importantly, unlike \cite{hu2016bandit}, we study stochastic non-convex optimization problems with the biased gradient oracles mentioned before. 

In \cite{balasubramanian2018zeroth}, the authors derive a non-asymptotic bound for a zeroth-order variant of the stochastic conditional gradient algorithm under an oracle model similar to \ref{as:biased_gs_estimation_error}, except that the estimation error component is absent. Specializing our results to have zero-mean estimation error would make our bounds comparable to those in \cite{balasubramanian2018zeroth}.
In \cite{nguyen2021inexact}, the authors consider a oracle model with a batch size parameter, and propose an algorithm that estimates the gradient on a mini-batch of sufficient size. They provide a sample complexity bound of $ \mathcal{ O}(1/ \epsilon^2 ) $ for the case of a convex function, and this matches our bound for the oracle \ref{as:biased_gs_estimation_error}.
Our bounds for the convex case are for the `practically preferred' last iterate, while those in \cite{nguyen2021inexact} are for an iterate chosen uniformly at random. In addition, our step-size choice does not require the knowledge of the underlying smoothness parameter.


In \cite{pasupathy2018sampling}, which is another closely related work, the authors study stochastic gradient methods in a setting where the objective is estimated using batch data, and a batch size of $m$ leads to an estimation error of $O\left(1/m^\alpha\right)$. 
Our biased gradient oracle framework is comparable to their setting, when $ \alpha = 1/2 $, and for this case, our non-asymptotic bounds match their asymptotic rate. 

The rest of the paper is organized as follows: Section \ref{sec:oracles} formulates the biased gradient oracles, along with motivating applications. Section \ref{sec:snco} considers the  stochastic non-convex optimization problem, and presents  non-asymptotic bounds for a randomized gradient descent algorithm with inputs from a biased gradient oracle. Section \ref{sec:sco} considers the stochastic convex optimization problem and presents non-asymptotic bounds for a stochastic gradient descent algorithm. Section \ref{sec:rl} highlights the applicability of our biased gradient oracles in a risk-sensitive reinforcement learning setting. Section \ref{sec:proofs} provides the proofs of all the bounds which are presented in the paper, and finally, Section \ref{sec:concl} provides the concluding remarks.

\textbf{Notation:} Throughout this paper we assume $\| \cdot \| = \| \cdot \| _ { 2 }$, and $ \mathbf{1}_{m \times n} $ is an $ m \times n $ matrix with each entry as one.

\section{Biased gradient oracles}
\label{sec:oracles}

Consider the following optimization problem:
\begin{align}
	\min_{x \in \R^d}  f(x), \label{eq:pb}
\end{align}
where the function $ f : \R^d \rightarrow \R $ is assumed to be smooth.
Gradient-based methods are very popular for solving the optimization problem formulated above, and we consider an iterative algorithm which obtains estimates of $ \nabla f(\cdot) $ through calls to a biased gradient oracle. We define two such oracles  below, and  subsequently, we provide motivating applications featuring biased function measurements.
 \subsection{Oracle definitions}
\begin{enumerate}[leftmargin=2.5em,label=(\textbf{O1})]
	\item \label{as:biased_sp_estimation_error} \textbf{Biased gradient oracle}\\
	\textit{Input:} $ x \in \R^d $, perturbation constant $ \eta > 0$, and batch size $ m >0 $.\\
	\textit{Output:} a gradient estimate $ g(x, \xi,\eta, m) \in \R^d$ that satisfies
	\setlist{nolistsep}
	\begin{enumerate}[noitemsep]
		\item \label{as:biased_estimation_error_a}$ \| \mathbb { E }_{\xi}   \left[ g \left( x  , \xi, \eta, m \right) \right] - \nabla f \left( x  \right) \|_{\infty} \leq  c_1 \eta^2 + \frac{c_3}{\eta \sqrt{m}}$,
		\item \label{as:biased_estimation_error_b} $  \mathbb { E }_{\xi} \big[ \left\| g \left( x  , \xi, \eta, m  \right) -  \mathbb { E }_{\xi}   \left[ g \left( x  , \xi, \eta, m  \right) \right]  \right\|^{2} \big]$ \\ $ \leq \frac{c_2}{\eta^{2}}  ,$
	\end{enumerate}
	for some constants $ c_1,c_2,c_3 > 0$.
\end{enumerate}
In the oracle defined above, the parameter $\eta$ is used to tradeoff bias and variance in the gradient estimates, while the parameter $m$ is motivated by practical models where exact function measurements are unavailable. Instead, one could choose larger values of $m$ to increase the accuracy of the function measurements. 

Next, we present an alternative to \ref{as:biased_sp_estimation_error}, where the bias of the gradient estimates can be reduced without adversely affecting the variance.

\begin{enumerate}[leftmargin=2.8em,noitemsep,label=(\textbf{O2})]
	\item \label{as:biased_gs_estimation_error} \textbf{Biased gradient oracle - variant}\\
	\textit{Input:} $ x \in \R^d $, perturbation constant $ \eta > 0$ and batch size $ m >0 $.\\
	\textit{Output:} a gradient estimate $ g(x, \xi, \eta, m) \in \R^d$ that satisfies
	\setlist{nolistsep}
	\begin{enumerate}[noitemsep]
		\item \label{as:biased_gs_estimation_error_a}$ \| \mathbb { E }_{\xi}   \left[ g \left( x , \xi ,\eta, m \right) \right]  - \nabla f \left( x\right)\|_{\infty} \leq c_1 \eta + \frac{c_3}{\eta \sqrt{m}}$,
		\item \label{as:biased_gs_estimation_error_b} $  \mathbb { E }_{\xi} \left[ \left\| g \left( x , \xi, \eta, m  \right) - \mathbb { E }_{\xi}   \left[ g \left( x , \xi, \eta, m  \right) \right]  \right\|^{2} \right]$ \\ $  \leq {c_2}{\eta^{2}}  + \tilde{c_2} ,$
	\end{enumerate}
	for some positive constants $ c_1, c_2, \widetilde{c_2}$ and $c_3$.
\end{enumerate}
\subsection{Illustrative applications}
\begin{example}
In the regular \textit{simulation optimization} setting \cite{fu2015handbook}, we are given function measurements with zero-mean noise, i.e., $ f(x) = \E_\xi[F(x,\xi)]$.
In contrast, we consider a model where the function measurements have an error term with positive mean. In this model, the objective $f$ is obtained as a solution to the following sub-problem over the optimization variable $y$ that belongs to a convex and compact set $\Y$ :
\begin{align}
 f(x) = \min_{y\in \Y} \E_\xi[H_x(y,\xi)] , \forall x\in \R^d. \label{eq:biased-opt1}
\end{align}
Owing to computational considerations, the sub-problem defined above cannot be solved exactly. Instead, an optimization algorithm can obtain inexact measurements $F(x,m)$ defined by
\begin{align}
	F(x,m) = \min_{y\in \Y} \E_\xi[H_x(y,\xi)] + \epsilon(m) , \forall x\in \R^d. \label{eq:biased-opt-est}
\end{align}
In the above, $m$ is the batch size parameter, and $\epsilon$ is a `positive' estimation error term. Choosing a larger batch size $m$ implies the sub-problem in \eqref{eq:biased-opt1} can be solved more precisely, leading to lower estimation error $\epsilon(m)$. 
\end{example}

The oracles \ref{as:biased_sp_estimation_error}--\ref{as:biased_gs_estimation_error} are also appealing in practical applications where  the objective $ f $ has to be estimated from i.i.d. samples coming from a r.v. $ X $, and the estimation scheme is biased. Estimation of risk measures such as CVaR is an example of an application where the de facto estimation scheme is biased. 
We shall illustrate the applicability of our oracle in the context of CVaR objective below.
\begin{example}
We begin by defining the VaR $V_\alpha(X)$ and CVaR $C_\alpha(X)$, at a pre-specified level $\alpha\in (0,1)$ below.
\begin{align*}
	V_\alpha(X) & = \inf \lbrace \xi : \prob{X \leq \xi} \geq \alpha \rbrace, \textrm{ and ~} \\
	C_{\alpha}(X)  &  =  V_{\alpha}(X)  + \frac{1}{1 - \alpha} \mathbb{E} \left[ X - V_{\alpha}(X) \right] ^+,
\end{align*}
where $[X]^+ = \max (0, X).$  If the distribution underlying $ X $ is continuous, then 
$C_{\alpha}(X)   =  \E [X | X \geq V_\alpha(X) ]$.

We now describe a well-known estimate of CVaR using $m$ i.i.d. samples $ \{X_i, i = 1,\ldots,m\}$.
Notice that CVaR estimation requires an estimate of VaR. Let $\hat{V}_{m, \alpha}$ and $\hat{C}_{m, \alpha}$ denote the estimates of VaR and CVaR. These quantities are defined as follows (see \citep{serfling2009approximation}): 
\begin{align}
	\hat{V}_{m, \alpha} & = X_{\left[  \lfloor m\alpha \rfloor \right]}, 
	\hat{C}_{m, \alpha}  = \frac{1}{m} \sum_{i=1}^m \frac{X_i \indic{X_i \geq \hat{V}_{m, \alpha}}}{(1-\alpha)} . \label{eq:cvar-estimate}
\end{align} 
In the above, $X_{[i]}$ denotes the $i$th order statistic, $\forall i$.
Notice that $\E\left(\hat{C}_{m, \alpha}\right) \ne C_\alpha(X)$, since the VaR estimate in \eqref{eq:cvar-estimate} is not unbiased. However,  a recent CVaR concentration result  in \citep{bhat2019concentration} shows that if the underlying r.v. $X$ is $\sigma$-sub-Gaussian\footnote{	A r.v. $ X $ is said to be $ \sigma $-sub-Gaussian for some $ \sigma > 0 $ if
	$	\mathbb{E}[\exp (\lambda X)] \leq \exp \left(\frac{\lambda^{2} \sigma^{2}}{2}\right), \text { for any } \lambda \in \mathbb{R}.
	$
},  then, for any $\epsilon > 0$, the following inequality holds:
\begin{equation}
	\mathbb{P}(|\hat{C}_{m, \alpha}-C_{\alpha}(X)|>\epsilon) \leq c_1 \exp (-c_2 m\epsilon^{2} (1-\alpha)^{2}), \label{eq:cvar_sub_gaussian}
\end{equation}
\sloppy
where constants $ c_1, c_2 $ depend on $ \sigma $.
Using \eqref{eq:cvar_sub_gaussian},
we have  
\begin{align}
	&\mathbb { E } \left|\hat{C}_{m, \alpha}-C_{\alpha}(X)\right| \nonumber\\
	&= \int_{0}^{\infty} \mathbb{P} ( |\hat{C}_{m, \alpha}(X)-C_{\alpha}(X) |>\epsilon ) d\epsilon
	\leq \frac{c_3}{\sqrt{m}},\label{eq:cvar-expec-bd}
\end{align}
where $ c_3>0 $ is an absolute constant.
\end{example}

In both examples illustrated above, the common element is biased function measurements. Using such measurements, one could construct gradient estimates using the simultaneous perturbation (SP) method \cite{bhatnagar-book}. We make this construction precise below.

Let $y^{+}(m)=f\left(x+\eta \Delta\right)+\xi^{+}(m)$,  and  $y^{-}(m)=f\left(x-\eta \Delta\right)+\xi^{-}(m)$.
Here $ \xi^{\pm}(m) $ are the estimation errors assuming a batch size of $m$, $ \eta $ is a perturbation constant, and $\Delta=\left(\Delta^{1}, \ldots, \Delta^{d}\right)^{\top}$  is a $d$-dimensional standard Gaussian vector. For the two examples discussed above, it is apparent that the estimation error is $\O(\frac{1}{\sqrt{m}})$ in expectation, if $m$ samples are used for estimation of $f$ at $(x\pm \eta \Delta)$ input parameters.

A gradient estimate is formed using two function evaluations (i.e., $ y^+ $ and $ y^- $) as follows:
\begin{align}
g(x, \xi, \eta, m) = \Delta \left[\frac{y^{+}(m) - y(m)}{\eta}\right],\label{eq:gs-est}
\end{align}
where $ \Delta $ is a $ d $-dimensional Gaussian vector composed of standard normal r.v.s. 
The estimate defined above is referred to as Gaussian smoothed functional, as well as Gaussian smoothing. This estimate  was proposed in \cite{katkul}, and studied later in a convex optimization setting in \cite{nesterov2011random}. A related estimate is random directions stochastic approximation (RDSA) with spherical perturbations, proposed in \cite{kushcla}.
 
 Assuming that the underlying function $f$ is three-times continuously differentiable, we have 
\begin{align*}
	&f(x \pm \eta \Delta) \\
	& = f(x) \pm \eta \Delta\tr \nabla f(x) + \frac{\eta^2}{2} \Delta\tr \nabla^2 f(x) \Delta +  \O(\eta^3).
\end{align*}
Hence,
 \begin{align*}
 	&\E\left[\Delta\left(\dfrac{f(x+\eta \Delta) - f(x-\eta \Delta)}{2\eta}\right) \right] \\
 	& =  \E\left[\Delta \Delta\tr \right] \nabla f(x)  + O(\eta^2) = \nabla f(x)  + \O(\eta^2),
 \end{align*}
where we used the fact that $\E\left[\Delta \Delta\tr \right] = I$, since $\Delta$ is a standard Gaussian vector.
Combining the equality above with the fact that the estimation error is $\O(\frac{1}{\sqrt{m}})$, we obtain
 \begin{align*}
	&\| \mathbb { E }_{\xi}   \left[ g \left( x  , \xi, \eta, m \right) \right] - \nabla f \left( x  \right) \|_{\infty} 
	 \le   c_1 \eta^2 + \frac{c_2}{\sqrt{m}},
\end{align*}
for some constants $c_1,c_2$. This satisfies the requirement (a) in \ref{as:biased_sp_estimation_error}. The requirement in (b) can be shown by squaring the estimator $g(\cdot,\cdot,\cdot,\cdot)$, assuming the square of the objective $f$ is bounded, leading to an inverse scaling with the perturbation constant $\eta^2$.   

A similar argument works for the case of a convex and smooth objective as well. In addition, a variety of distributions can be employed for the random perturbations, cf. \cite{spall1992multivariate, prashanth2017rdsa, prashanth2018random, bhatnagar-book,hu2016bandit}. 

The oracle variant defined in \ref{as:biased_gs_estimation_error} can also be constructed using the estimator in \eqref{eq:gs-est}. In particular, following arguments used in Lemma 3 in \cite{nesterov2011random} together with the fact that estimation error is of order $\O(\frac{1}{\sqrt{m}})$ would lead to the condition (a) in \ref{as:biased_gs_estimation_error}. 
As mentioned before, in a regular simulation optimization setting, 	one observes a sample $F(x,\xi)$ that satisfies $\E [F(x,\xi)] = f(x)$, $\forall x$. If the noise $\xi$ has bounded variance, then one can ensure condition (b) in \ref{as:biased_sp_estimation_error}, where the variance of the estimator scales inversely with $\eta^2$. However, under additional assumptions, such as $\nabla f(x) = \E [\nabla F(x,\xi)]$, one can get rid of the inverse scaling with $\eta$, avoiding the bias-variance tradeoff through the parameter $\eta$. The requirement  $\nabla f(x) = \E\left[ \nabla F(x,\xi)\right]$ amounts to an interchange of differentation and integration operators, usually by invoking the dominated convergence theorem, and is common in perturbation analysis, cf. \cite{asmussen2007stochastic, glasserman1991gradient}. 
For a proof that leads to condition (b) in \ref{as:biased_gs_estimation_error}, a straightforward variation of the proof of Lemma B.1 in \cite{balasubramanian2018zeroth}, which incorporates biased function measurements can be worked out, and we omit the details.

\subsection{Performance metrics}
We consider stochastic gradient (SG) type algorithms for solving \eqref{eq:pb}, with inputs from either \ref{as:biased_sp_estimation_error} or \ref{as:biased_gs_estimation_error}. 
A SG algorithm runs for $N$ iterations, and outputs a point $x_R$, that could be chosen randomly from the iterates $x_1,\ldots, x_{N}$. 
We study SG schemes in following two contexts:  (i) the case when the objective $f$ is convex; and (ii) the case when the objective $f$ is not assumed to be convex. 
 In  case (i), we provide bounds on the optimization error, i.e., $ f(x_R) - f(x^*) $, where $x^*$ is a minima of $f$. On the other hand, in case (ii), i.e., when the objective is non-convex, it is difficult to bound the optimization error. A popular alternative is to establish local convergence. i.e., to a point where the gradient of the objective is small (cf.  \cite{ghadimi2013stochastic,bottou2018optimization}). The following definition makes the optimization objective apparent in both cases.

\begin{definition}
	\label{def:esolution} 
	Let $ x_N \in \R^d $ be the output of the algorithm and $ \epsilon > 0 $ be a target accuracy, then: 
	\begin{enumerate}[label=(\roman*)]
		\item If $ f $ is non-convex, $ x_N $ is called an $ \epsilon $-stationary point of \eqref{eq:pb}, if $ 	\mathbb{E} \left\| \nabla f \left( x _ { N } \right) \right\| ^ { 2 } \le \epsilon $.
		\item If $ f $ is convex, $ x_N $ is called an $ \epsilon $-optimal point of \eqref{eq:pb}, if $ 	\mathbb{E} [f \left( x _ { N } \right)] - f(x^*) \le \epsilon $, where $ x^* $ is an optimal solution to \eqref{eq:pb}.
	\end{enumerate}
\end{definition}

The SG algorithms are judged using the iteration as well as sample complexity, which are defined below. 
\begin{definition}
	\label{def:iter_complexity}
	The iteration complexity of an algorithm $\mathcal A$ is the number of calls $\mathcal A$ makes to a biased gradient oracle before finding an $\epsilon$-stationary  (resp. $\epsilon$-optimal) point for a  non-convex (resp. convex) objective function.  
\end{definition}

\begin{definition}
	\label{def:sample_complexity}
	Suppose an algorithm $\mathcal A$ makes $N$ calls to a biased gradient oracle before finding an $\epsilon$-stationary  (resp. $\epsilon$-optimal) point for a  non-convex (resp. convex) objective function. Then, the sample complexity of $\mathcal A$ is $ \sum_{ i = 1 }^{N}  m_i$, where $ m_i, i=1,\ldots,N,$ is the batch size in iteration $ i $. 
\end{definition}

\section{Stochastic Non-convex Optimization}
\label{sec:snco}

In this section, we consider the problem in \eqref{eq:pb}, with an objective $f$ that is smooth, but not necessarily convex. We analyze the non-asymptotic performance of the RSG algorithm  \citep{ghadimi2013stochastic}, with inputs from a biased gradient oracle (either \ref{as:biased_sp_estimation_error} or \ref{as:biased_gs_estimation_error}). 
The pseudocode for the algorithm is given below. This algorithm performs an incremental update as defined in \eqref{eq:rsg}, and outputs a random iterate, after $ N $ iterations\footnote{The bounds in Section \ref{sec:snco} are for a random iterate $x_R$, where $ R $ is uniformly distributed over $ \{1,\dots,N \} $, and the expectation is taken with respect to $ R $ and noise $\xi_{[N]}:=\left(\xi_{1}, \ldots, \xi_{N}\right) $.
 }. 
\begin{algorithm}[h]
	\caption{Randomized stochastic gradient algorithm with a biased gradient oracle (RSG-BGO)}
	\label{alg:rsg}
	\begin{algorithmic}
		\State {\bfseries Input:} Initial point $ x_1 \in \R^d $, iteration limit $ N $, stepsizes $ \gamma_{ k },$ perturbation constant $ \eta_k $, batch size  $  m_k$, and probability mass function $ P_R(\cdot) $ supported on $ \{1,\dots,N\} $ (Let $ R $ denote the corresponding random variable).
		\For{$k =  1,\dots,R$}
				\State Call the oracle \ref{as:biased_sp_estimation_error} or \ref{as:biased_gs_estimation_error} with $x_k,\eta_k$ and $ m_k $, to obtain the gradient estimate $g_k $.
       \State Perform the following stochastic gradient update:
		\begin{align}
		x _ { k + 1 } = x _ { k } - \gamma _ { k} g \left( x _ { k } , \xi _ { k }, \eta_k, m_k \right). \label{eq:rsg}
		\end{align}
		\EndFor
		\State {\bfseries Return} $x_R$.
	\end{algorithmic}
\end{algorithm}

For the non-asymptotic analysis of the RSG algorithm, we make the following assumptions:
\begin{enumerate}[leftmargin=2.5em,noitemsep,label=(\textbf{A1})]
	\item \label{as:lipschitz}
	Function $ f$ has Lipschitz continuous gradient with constant $ L > 0 $, i.e.,
	\[	\| \nabla f ( x ) - \nabla f ( y ) \|  \leq L \| x - y \|,
	\quad \forall x , y \in \mathbb { R } ^ { d }.\]
\end{enumerate}

\begin{enumerate}[leftmargin=2.5em,label=(\textbf{A2})]
	\item \label{as:boundedness} 
	There exists a constant $ B > 0 $ such that $\| \nabla f ( x ) \|_1 \leq B, \forall x \in \R^d$.
\end{enumerate}
The smoothness assumption in \ref{as:lipschitz} is standard in the analysis of gradient-based algorithm (cf. \cite{ghadimi2013stochastic,balasubramanian2018zeroth}).
The boundedness requirement in \ref{as:boundedness} is made in the context of zeroth-order optimization in \cite{balasubramanian2018zeroth}, and can be inferred from the assumptions common to the analysis of policy-gradient algorithms in a reinforcement learning context, cf. \cite{shen2019hessian}. 

We provide below a non-asymptotic bound for RSG-BGO algorithm with \ref{as:biased_sp_estimation_error}.
\begin{theorem} (\textbf{RSG-BGO under \ref{as:biased_sp_estimation_error}}) \label{thm:biased_sp_esterr}\ \\
	Assume \ref{as:lipschitz} and \ref{as:boundedness}. With the oracle \ref{as:biased_sp_estimation_error}, suppose that the RSG-BGO algorithm is run with the stepsize $ \gamma_k $ and perturbation constant $ \eta_k $ set as follows $ \forall k \geq 1 $:
	\begin{align} 
		\gamma_{k} =  min \bigg\{\frac{1}{L}, \frac{\gamma_0}{N^{2/3}}\bigg\}, \text{ }  \eta_k = \frac{\eta_0}{N^{1/6}}, \label{eq:biased_sp_par}
	\end{align}
	for some constant $ \gamma_0, \eta_0 > 0. $\\
\noindent\textbf{(i)} If the batch size $ m_k = m_0 N,  \forall k \geq 1 $, for some constant $ m_0 > 0 $, then, for any $ N \ge 1 $, we have
\begin{align} 
	&	\mathbb { E } \left\| \nabla f \left( x _ { R } \right) \right\| ^ { 2 } 
	\le  \frac{ 2 L D_f }{{ N }} + \frac{\mathcal{Z}_1}{N^{1/3}}   +  \frac{\mathcal{Z}_2}{N^{1/3}}, \label{eq:zrsg_o2_i}
\end{align}
where $\mathcal{Z}_1 = \frac{ 2D_f }{\gamma_0} + {4 B c_1 \eta_0^2}+  \frac{L d c_1^2 \gamma_0 \eta_0^4 }{ N} + \frac{L  c_2 \gamma_0 }{\eta_0^2 }$, $ \mathcal{Z}_2 = \frac{4 B c_3}{\eta_0 \sqrt{m_0}} + \frac{L d \gamma_0 }{N} \left( \frac{2 c_1 c_3 \eta_0}{\sqrt{m_0}} +  \frac{c_3^2}{\eta_0^2 m_0}  \right)$, constants $c_1 $, $ c_2 $ and $ c_3 $ are as defined in \ref{as:biased_sp_estimation_error}, $ B $ is as defined in \ref{as:boundedness},
\begin{align}
	D_f =  f(x_1) - f(x^*),  \label{eq:D_f}
\end{align}
and $ x^* $ is an optimal solution to \eqref{eq:pb}.\\[1ex]
\noindent\textbf{(ii)} If the batch size $ m_k = m_0 k^\beta, \forall k \geq 1 $, for some constant $ \beta \in (0, 1) $ and $ m_0 > 0 $, then, for any $ N \ge 1 $, we have
\begin{align}
	&\mathbb { E } \left\| \nabla f \left( x _ { R } \right) \right\| ^ { 2 }   \leq \frac{ 2 L D_f }{{ N }} +  {\frac{ L \gamma_0 d  c_3^2 }{ \eta_0^2 m_0 N^{\frac{3\beta+1}{3}}  (1 -\beta)}}\nonumber \\
	& +   {\frac{4 B c_3  }{ \eta_0 \sqrt{m_0} N^{\frac{3\beta - 1}{6} }  ( 1 -\tfrac{\beta}{2} )}}  + {\frac{ 2 L d \gamma_0 \eta_0 c_1 c_3 }{\sqrt{m_0} N^{\frac{3\beta + 5}{6}} ( 1 - \tfrac{\beta}{2} )}}, \label{eq:zrsg_o2_ii} 	
\end{align}

	where constants are the same as in part (i).
\end{theorem}
\begin{proof}
	See Section \ref{pf:biased_esterr}.
\end{proof}

\begin{remark}
	The overall rate, from the bound above, is $\mathcal{ O}\left(N^{-1/3}\right)$, and this is not surprising because the bias of the gradient cannot be made arbitrarily small by setting $\eta$ to a low value, as the variance of the gradient estimates scales inversely with $\eta$. The (asymptotic) convergence rate results for simultaneous perturbation stochastic approximation (SPSA) in \cite{spall1992multivariate}, and RDSA in \cite{prashanth2017rdsa}, also exhibit the same order, under an oracle that is a variant to \ref{as:biased_sp_estimation_error} (without the estimation error component). 
	\end{remark}

Using the bound in Theorem \ref{thm:biased_sp_esterr}, it is easy to see that the iteration complexity is $ \mathcal{ O} (\frac{1}{\epsilon^3})$, and the sample complexity is $ N^2 $.

\begin{remark}
	For understanding the dimension dependence in the iteration complexity, let us consider the special case where \ref{as:biased_sp_estimation_error} is implemented using either SPSA \cite{spall1992multivariate} or RDSA \cite{prashanth2017rdsa}. In this case, $ c_1 = \kappa_1 d^3$ and $ c_2 = \kappa_2 d $, where $ \kappa_1, \kappa_2 > 0$ are dimension-independent constants. Choosing $ \gamma_0 = d^{{-4}/{3}}$, $ \eta_0 = d^{{-5}/{6}}$ and $ m_0 = 1 $ in \eqref{eq:biased_sp_par}, the overall iteration complexity of RSG-BGO turns out to be $ \mathcal{ O} (\frac{d^{4}}{\epsilon^3})$. 
\end{remark}	

We provide below a non-asymptotic bound for the RSG-BGO algorithm with \ref{as:biased_gs_estimation_error}.
\begin{theorem} (\textbf{RSG-BGO under \ref{as:biased_gs_estimation_error}}) \label{thm:biased_gs_esterr}\ \\
	Assume \ref{as:lipschitz} and \ref{as:boundedness}. With the oracle \ref{as:biased_gs_estimation_error}, suppose that the RSG-BGO algorithm is run with the stepsize $ \gamma_k $, perturbation constant $ \eta_k $, and batch size $ m_k $ set as follows $\forall k \geq 1$:
	\begin{align} 
		\gamma_{k} =  min \bigg\{\frac{1}{L}, \frac{\gamma_0}{\sqrt{N}}\bigg\},  \eta_k = \frac{\eta_0}{\sqrt{N}}, \text{} m_k = m_0  N^{2}, \label{eq:biased_gs_par}
	\end{align}
	for some constant $ \gamma_0, \eta_0, m_0 > 0. $\\
	Then, for any $ N \ge 1 $, we have
	\begin{align} 
		& \mathbb { E } \left\| \nabla f \left( x _ { R } \right) \right\| ^ { 2 }    \le  \frac{ 2 L D_f}{{ N }}   + \frac{\mathcal{Z}_3}{\sqrt{N}}, \nonumber
	\end{align}
where $ \mathcal{Z}_3 =  \frac{ 2 D_f}{\gamma_0} +4 B \mathcal{ Z}_4  + L \gamma_0 \left(\frac{ d \mathcal{ Z }_4^2}{N}  +  \frac{  c_2 \eta_0^2}{N } + {\tilde{c_2}} \right) $, $ \mathcal{ Z }_4 = {c_1 \eta_0 } + \frac{c_3}{\eta_0 \sqrt{m_0}} $, constants $ c_1 $, $ c_2 $, $ \tilde{c_2} $ and $ c_3 $ are as defined in \ref{as:biased_gs_estimation_error}, $ B $ is as defined in \ref{as:boundedness}, and $ D_f $ is as defined in \eqref{eq:D_f}.
\end{theorem}
\begin{proof}
	See Section \ref{pf:thm_biased_gs_esterr}.
\end{proof}

From the bound in Theorem \ref{thm:biased_gs_esterr}, it is easy to see that the iteration complexity is $ \mathcal{ O} (\frac{1}{\epsilon^2})$, and the sample complexity is $  N^3 $. 
This bound is better than the corresponding bound with \ref{as:biased_sp_estimation_error}, we believe this improvement is because the variance of the gradient estimate in \ref{as:biased_gs_estimation_error} does not increase when bias is reduced. 

\begin{remark}
	In \cite{ghadimi2013stochastic}, the authors derive a non-asymptotic bound for a zeroth-order variant of their RSG algorithm under an oracle that is a variant to \ref{as:biased_gs_estimation_error} (without the estimation error component). Our result in Theorem \ref{thm:biased_gs_esterr} matches their bound. Moreover, unlike \cite{ghadimi2013stochastic}, we derive a non-asymptotic bound  for the oracle \ref{as:biased_gs_estimation_error}, which involves an estimation error component. 
	
	An advantage with our analysis is that it allows a simpler distribution for picking the final iterate (see Proposition \ref{prop:biased_esterr} in the Section \ref{pf:biased_esterr}). In particular, our bounds hold for an iterate $x_R$ that is picked uniformly at random from $\{x_1,\ldots,x_N\}$. The net effect is that of iterate averaging, except that the averaging happens in expectation.
\end{remark}

\begin{remark}
	For understanding the dimension dependence in the iteration complexity, let us consider the special case where \ref{as:biased_gs_estimation_error} is implemented using the Gaussian smoothing approach \cite{nesterov2011random, balasubramanian2018zeroth}. In this case, $c_1 = {\frac{L (d + 3)^{\frac{3}{2}}}{2}},  c_2 = { \frac{L^2 (d + 3)^3}{2}},  \widetilde{c_2} = 2(d+5)(B^2 + {\sigma^2}),$
	where $ \sigma^2 $ is the bound on variance of the estimator of $ f(x) $. 
	Choosing the stepsize $ \gamma_0 = d^{-1/2}$, $ \eta_0 = d^{-1}$ and $ m_0 = d $ in \eqref{eq:biased_gs_par}, the overall iteration complexity of RSG-BGO turns out to be $ \mathcal{ O} (\frac{d}{\epsilon^2})$. 
\end{remark}

\section{Stochastic Convex Optimization}
\label{sec:sco}
In this section, we consider the problem in \eqref{eq:pb}, under the assumption that $f$ is a convex function.
Let $ x^* \in \R^d$ be a minimizer of the objective $ f $. We first analyze the RSG-BGO algorithm in a convex setting, and subsequently present the SGD-BGO algorithm.

\subsection{Randomized stochastic gradient algorithm with a biased gradient oracle (RSG-BGO)} \label{sec:bounds_rsg_cvx}

We provide below a non-asymptotic bound for RSG algorithm with \ref{as:biased_sp_estimation_error}.

\begin{theorem}(\textbf{RSG-BGO under \ref{as:biased_sp_estimation_error}})  \label{thm:biased_sp_esterr_convex} \ \\
	Assume \ref{as:lipschitz}. With the oracle \ref{as:biased_sp_estimation_error}, suppose that the RSG-BGO algorithm is run with the batch size $ m_k  = m_0 N, \forall k \geq 1$, for some constant $ m_0 > 0 $ and stepsize $ \gamma_k $, perturbation constant $ \eta_k $ set as defined in \eqref{eq:biased_sp_par}.
	Then, for any $ N \ge 1 $, we have
	\begin{align*} 
		&	\mathbb { E } \left[ f \left( x _ { R } \right) \right]  - f (x^ { * })
		\le \frac{  L D^2}{{ N }} + \frac{\mathcal{K}_1}{N^{1/3}},
	\end{align*}
	where $ \mathcal{K}_1 =   \frac{ D^2}{\gamma_0}  + 4 \sqrt{d} D \mathcal{ K}_2 + \frac{ \gamma_0 d \mathcal{ K}_2^2 }{N }  + \frac{ \gamma_0 c_2 }{\eta_0^2} $, $ \mathcal{ K}_2 = ({c_1 \eta_0^2 }  + \frac{c_3}{\eta_0 \sqrt{m_0}} )$ constants $ c_1 $, $ c_2 $ and $ c_3 $ are as defined in \ref{as:biased_sp_estimation_error}, 
	\begin{align}
		D = \| x_1 - x^*\|,  \label{eq:dia}
	\end{align}
	and $ x^* $ is an optimal solution to \eqref{eq:pb}.
\end{theorem}
\begin{proof}
	See Section \ref{pf:thm_biased_sp_esterr_convex}.
\end{proof}

The  iteration complexity of $ \mathcal{ O} (\frac{1}{\epsilon^3})$ and the sample complexity $ N^2 $ of the RHS above matches that in Theorem \ref{thm:biased_sp_esterr} with the non-convex objective. However, unlike non-convex case, we bound the optimization error, i.e., $ \mathbb{E} [f(x_R)] - f(x^*) $. 

We now provide a non-asymptotic bound for the RSG algorithm with \ref{as:biased_gs_estimation_error} for the convex objective.
\begin{theorem} (\textbf{RSG-BGO under \ref{as:biased_gs_estimation_error}}) \label{thm:biased_gs_esterr_convex}\ \\
	Assume \ref{as:lipschitz}. With the oracle \ref{as:biased_gs_estimation_error}, suppose that the RSG-BGO algorithm is 
	run with the stepsize $ \gamma_k $, perturbation constant $ \eta_k $, and batch size $ m_k $ set as defined in \eqref{eq:biased_gs_par}.
	Then, for any $ N \ge 1 $, we have	
\begin{align*} 
	\mathbb { E } \left[ f \left( x _ { R } \right) \right] - f (x^ { * })  \le \frac{  L D^2}{{ N }} + \frac{\mathcal{K}_3}{\sqrt{N}},
\end{align*}	
where $ \mathcal{K}_3 =  \frac{ D^2}{\gamma_0} + 4 \sqrt{d} D \mathcal{ K}_4 + \frac{ \gamma_0 d \mathcal{ K}_4^2 }{N }  + \frac{  \gamma_0 \eta_0^2 c_2  }{N} + \gamma_0 {\tilde{c_2}}$, $\mathcal{ K}_4 = {c_1 }{\eta_0} + \frac{c_3}{\eta_0 \sqrt{m_0}} $, constants $ c_1 $, $ c_2 $, $ \tilde{c_2} $ and $ c_3 $ are as defined in \ref{as:biased_gs_estimation_error}, and $ D $ is as defined in \eqref{eq:dia}.
\end{theorem}

\begin{proof}
	See Section \ref{pf:thm_biased_gs_esterr_convex}.
\end{proof}

From the bound in Theorem \ref{thm:biased_gs_esterr_convex}, it is easy to see that the iteration complexity is $ \mathcal{ O} (\frac{1}{\epsilon^2})$, and the sample complexity is $ N^3 $.

\subsection{Stochastic gradient descent algorithm with a biased gradient oracle (SGD-BGO)} \label{sec:bounds_sgd_cvx}
In this section, we study a stochastic gradient descent algorithm with a biased gradient oracle. Unlike RSG-BGO algorithm whose bounds were for a random iterate, the bounds that we derive for SGD-BGO are for the last iterate, which is the preferred point in practical implementations. 
The pseudocode for the SGD-BGO algorithm with inputs from a biased gradient oracle is given in Algorithm \ref{alg:net}.

\begin{algorithm}[h]
	\caption{Stochastic gradient descent algorithm with a biased gradient oracle (SGD-BGO)}
	\label{alg:net}
	\begin{algorithmic}
		\State {\bfseries Input:} Initial point $ x_1 \in \R^d $, iteration limit $ N $, stepsizes $ \gamma_{ k },$ perturbation constant $ \eta_k $ and batch size  $  m_k$.
		\For{$k =  1,\dots,N$}
		\State Call the oracle \ref{as:biased_sp_estimation_error} or \ref{as:biased_gs_estimation_error} with $x_k,\eta_k$ and $ m_k $, to obtain the gradient estimate $g_k $.
		\State Perform the following stochastic gradient update:
		\begin{align}
		x_{k+1} =  x_{k} - \gamma_k g(x_k, \xi_{ k }, \eta_k, m_k),\label{eq:zsgd-convex}
		\end{align}		
		\EndFor
		\State {\bfseries Return} $x_N$.
	\end{algorithmic}
\end{algorithm}

Following the approach from \cite{jain2019making}, we assume the knowledge of the the total number of iterations $N$ and split the horizon $N$ into $l$ phases. The choice of phase lengths, and the step-size decay in each phase is performed along the lines of \cite{jain2019making}. However, unlike their work that assumed unbiased gradient information, we operate in a setting where biased gradient information is available through \ref{as:biased_sp_estimation_error} or \ref{as:biased_gs_estimation_error}, and this induces significant deviations in the proof. Morever, our setting features a perturbation constant parameter, which has to be chosen in a phase-dependent manner as well. 

We make the choice of phases precise below.
\begin{align} 
\text{Let  }   l  := &\inf\{i : N\cdot 2^{-i} \leq 1\}, \nonumber\\
  N_i := & N - \lceil N\cdot 2^{-i}\rceil,\ 0\leq i\leq l,\mbox{ and } N_{l+1} := N.\label{eq:def_Ni}
  \end{align}  
  From the phase definitions above, it can be seen that $N_i$ is an increasing sequence. Further, $N_1 \approx \frac{N}{2}, N_2 \approx \frac{N}{2} + \frac{N}{4},$ and so on.
  In the result below, we provide a non-asymptotic bound on the optimization error, i.e., $ \E [f(x_N)] - f(x^*) $ for the SGD-BGO algorithm under \ref{as:biased_sp_estimation_error}.
  
\begin{theorem} (\textbf{SGD-BGO under \ref{as:biased_sp_estimation_error}}) \\
	Assume \ref{as:boundedness}. With the oracle \ref{as:biased_sp_estimation_error}, suppose that the SGD-BGO algorithm is run with the stepsize $ \gamma_k $, perturbation constant $ \eta_k $, and batch size $ m_k $ set as follows:
\begin{align*}
	\gamma_k = \frac{\gamma_0 \cdot 2^{-i}}{N^{2/3}},  \text{ } \eta_k = \frac{ \eta_0 2^{-i/4}}{N^{1/6}},  \textrm{ and } m_k = 2^i N  ,
\end{align*} 
for some constant $ \gamma_0, \eta_0  >0 $, when $ N_i < k \leq N_{i+1}, 0\leq i \leq l $, with $ N_i, l $ as defined in \eqref{eq:def_Ni}. Then, for any $N\geq 4$, we have
	$$\mathbb{E}[f(x_N)] - f(x^*) \leq \frac{\mathcal{K}_5}{N^{1/3}},$$ \label{thm:convex_sp_esterr}
	where $ \mathcal{K}_5 =   \frac{4 D^2}{\gamma_0} + \frac{11 \gamma_0 {B}^2}{N^{1/3}} +  67 D \sqrt{d} \mathcal{ K}_6 + \frac{20 \gamma_0 \sqrt{d}   B \mathcal{ K}_6 }{N^{2/3}} + \frac{10 \gamma_0 d \mathcal{ K}_6^2}{N} + \frac{18 \gamma_0 c_2}{\eta_0} $, $ \mathcal{ K}_6 = (c_1\eta_0^2 + \frac{c_3}{\eta_0}) $, constants $c_1, c_2, c_3$ are as defined in \ref{as:biased_sp_estimation_error}, and $ D $ is as defined in \eqref{eq:dia}.
\end{theorem}
\begin{proof}
	See Section \ref{pf:biased_convex_sp_esterr}.
\end{proof}

From the bound in Theorem \ref{thm:convex_sp_esterr}, it is easy to see that the iteration complexity of SGD-BGO is $ \mathcal{ O} (\frac{1}{\epsilon^3})$, and the sample complexity is $ N^2 \log_2 N $.

\begin{remark}
The analysis used in arriving at the bounds in Theorem \ref{thm:convex_sp_esterr} cannot be extended to the non-convex case. This is because the analysis takes a dual viewpoint and approaches the minima of the objective from below, and in this process, convexity is strictly necessary. Intuitively, it may be challenging to provide bounds for the last iterate sans averaging in a non-convex optimization setting, while it is possible to provide bounds for the averaged iterate (or the random iterate of RSG-BGO, which is an average in expectation) in the non-convex case.
\end{remark}

\begin{theorem} (\textbf{SGD-BGO under \ref{as:biased_gs_estimation_error}})  \\
	Assume \ref{as:boundedness}.  With the oracle \ref{as:biased_gs_estimation_error}, suppose that the SGD-BGO algorithm is run with the stepsize $ \gamma_k $, perturbation constant $ \eta_k $ and batch size $ m_k $ set as follows:
	\begin{align}
		\gamma_k = \frac{\gamma_0 \cdot 2^{-i}}{\sqrt{ N}}, \quad \eta_k = \frac{\eta_0 2^{-i}}{ N},  \quad m_k = 2^{3i} N^3  ,
		\label{eq:weak_step_size_gs}
	\end{align} 
	for some constant $ \gamma_0, \eta_0  >0 $, when $ N_i < k \leq N_{i+1}, 0\leq i \leq l $, with $ N_i, l $ as defined in \eqref{eq:def_Ni}. Then, for any $N\geq 4$, we have
	$$\mathbb{E}[f(x_N)] - f(x^*) \leq \frac{\mathcal{K}_7}{\sqrt{N}},$$ \label{thm:convex_gs_esterr}
	where $ \mathcal{K}_7 =   \frac{4D^2}{\gamma_0} + {11 \gamma_0 {B}^2} + 39D \sqrt{d} \mathcal{ K}_8 +  \frac{20 \sqrt{d} B \gamma_0 \mathcal{ K}_8 }{\sqrt{N}}  + \frac{10 d \gamma_0 \mathcal{ K}_8^2 }{N}   + \frac{10 \gamma_0 \eta_0^2 c_2}{N^2} + {10 \gamma_0 \tilde{c_2}}$, $ \mathcal{ K}_8 = ( \frac{c_1 \eta_0}{\sqrt{N}}  + \frac{ c_3}{\eta_0}  )$, constants $ c_1 $, $ c_2 $, $ \tilde{c_2} $ are as in \ref{as:biased_gs_estimation_error}, and $ D $ is as defined in \eqref{eq:dia}.
\end{theorem}
\begin{proof}
	See Section \ref{pf:biased_convex_gs_esterr}.
\end{proof}
From the bound in Theorem \ref{thm:convex_gs_esterr}, it is easy to see that the iteration complexity of SGD-BGO is $ \mathcal{ O} (\frac{1}{\epsilon^2})$, and the sample complexity is $ 4N^4 (N^2 - 1) / 3 $. Further, it is interesting to note that the bound with \ref{as:biased_gs_estimation_error} matches up to constant factors, the bound obtained in \cite{jain2019making} for the case when unbiased gradient information is available.
Unlike \cite{ghadimi2013stochastic}, where the authors provide a $ \O(\frac{1}{\epsilon^2}) $ iteration complexity bound for 
a random iterate using the RSG-BGO algorithm, we provide bound for the last iterate of SGD-BGO.
Apart from a practical preference for using the last iterate, an advantage with our approach is that for setting the step size $ \gamma_k $ and perturbation constant $\eta_k$ \eqref{eq:weak_step_size_gs}, we do not require the knowledge of Lipschitz constant $ L $ (see \ref{as:lipschitz}) and $D_X := \| x_1 - x^* \|$. The latter quantity is typically unavailable in practice, as it relates to the initial error.

One could specialize the bounds in Theorems \ref{thm:biased_sp_esterr_convex}--\ref{thm:convex_gs_esterr} for the case when the underlying oracles are implemented using SPSA or GS methods, and we omit the details due to space constraints.

\section{Application: Risk-Sensitive Reinforcement Learning}
\label{sec:rl}
We consider a stochastic shortest path (SSP) problem, with a special cost-free absorbing state, say $0$. 
We restrict our attention to `proper' policies, which ensure state $0$ is recurrent, and the remaining states are transient in the Markov chain underlying the policy considered. We define an episode as a sample path $\{x^0, \ldots, x^\tau\}$, where  $x^\tau=0$, and $\tau$ is the first passage time to state $0$. 

Consider a smoothly parameterized class of policies $\{\pi_x \mid x\in \R^d \}$. Suppose that the policy $\pi_x$ is a continuously differentiable function of the parameter $x$: a standard assumption in policy gradient literature.  Let $K_x(x^0)$ denote the total cost r.v. under policy $x$ starting in state $x^0$, i.e., $K_x ( x^0 ) = \sum _ { t = 0 } ^ { \tau -1 } \gamma^t k(x_t,a_t)$, where $ 0  < \gamma < 1 $ is the discount factor and $ k(x_t,a_t) $ is the single-stage cost incurred at time instant $ t $ in state $x_t$ on choosing action $a_t$. Here actions $a_t$ are chosen according to policy $\pi_x$, which is parameterized by $x$. 

The classic objective in RL is to find a policy that minimizes, in expectation, the total cost.
We consider a risk-sensitive RL setting, where the goal is to find a policy that optimizes a certain risk measure, i.e., the following problem:
\begin{align}
	\min_{x \in \R^d} \left\{ \rho(K_{ x }(x^0))\right\},
	\label{eq:cvar_rl}
\end{align}
where $x \in \R^d$ parameterizes the policy $\pi_{ x }$, and $\rho$ is a risk measure. As examples for the risk measure, one could consider CVaR \cite{rocka}, utility-based shortfall risk (UBSR) \cite{follmer2002convex}, and spectral risk measure (SRM) \cite{acerbi2002spectral}. 

Notice that the optimization problem in \eqref{eq:cvar_rl} is non-convex in nature.
For solving the problem defined above using gradient-based methods, one requires (i) an estimate of the risk measure for any given policy $\pi_x$; and (ii) an estimate of the gradient of the risk measure w.r.t. the policy parameter $x$.  We elaborate on these two parts below.

We simulate $ m $ episodes simulated using the policy $ \pi_{ x } $, and collect samples of the total cost $K_{ x }(x^0)$. 
Using these samples, define the empirical distribution function (EDF) $F_m$ of $K_x ( x^0 )$ as follows $ F_m(x)=\frac{1}{m} \sum_{i=1}^m \indic{K_x ( x^0 ) \leq x}$, for any $x\in \R$. 
Using the EDF, we form the estimate $ \rho_{m} $ of $ \rho(K_{ x }(x^0)) $ as follows:
\begin{align} \rho_m = \rho(F_m).\label{eq:rho-est}\end{align}
Such as estimation scheme for an abstract risk measure has been considered earlier in \cite{cassel2018general,la2020concentration}. 

Next, we present a bound in expectation for the estimation error associated with \eqref{eq:rho-est}.
\begin{proposition}
	\label{prop:risk_measure}
	Suppose the risk measure $ \rho $ satisfies the following continuity requirement for any two distributions $F,G$:
	\begin{align} 	\left| \rho(F) - \rho(G)\right| \le L W_1(F,G), \label{eq:rho-cont}\end{align}
	where $W_1(F,G)$ is the Wasserstein distance\footnote{Given two cumulative distribution functions (CDFs) $ F $ and $ G $ on $ \R $, let $ \Gamma(F , G ) $ denote the set of all joint distributions on $ \R^2 $ having $ F $ and $ G $ as marginals. Then the Wasserstein distance between $ F $ and $ G $ is defined by
		$W_{1}\left(F, G\right) \triangleq\left[\inf _{C \in \Gamma\left(F, G\right)} \int_{\mathbb{R}^{2}}|x-y| d C(x, y)\right].$
	} between distributions $F$ and $G$, and $ L $ is as defined in \ref{as:lipschitz}. 
	Suppose the r.v. $K_x(x^0)$ satisfies $\E[ K_x(x^0)^2] \le B < \infty$, for any $x \in \R^d$. Then, 
	\[ \E \left| \rho_m - \rho(K_x(x^0)) \right| \le \frac{c}{\sqrt{m}},\]
	for some constant $c$ that depends on $B$.
\end{proposition}	
\begin{proof}
	See Section \ref{pf:prop_risk_measure}.
\end{proof}
The continuity requirement in \eqref{eq:rho-cont} is satisfied by the three popular risk measures CVaR, UBSR and SRM, and the reader is referred to \cite{la2020concentration} for details.

To construct an estimate of the gradient of the risk measure $\rho(K_x(x^0))$, one could employ the simultaneous perturbation method, e.g., the estimate in \eqref{eq:gs-est}. 
Using this gradient estimate and the template of RSG-BGO algorithm, we arrive at the following update iteration for the risk-sensitive policy gradient (risk-PG) algorithm: 
\[ x _ { k + 1 } = x _ { k } - \gamma _ { k} g_k, \]
where $\gamma_k$ is the stepsize and $g_k$ is an estimate of the gradient of the risk measure $\rho(K_{x_k}(x^0))$. 
To elaborate on the gradient estimation aspect of the algorithm above,
we first simulate $ m_k $ trajectories of the underlying MDP with policy parameters $x_k + \eta_k \Delta_k$ and $x_k - \eta_k \Delta_k$, respectively. Here $\eta_k$  is the perturbation constant, and $\Delta_k$ is a standard Gaussian vector.   Using \eqref{eq:rho-est}, we estimate the risk measures $\rho(K_{x_k \pm \eta_k \Delta_k}(x^0))$ corresponding to the aforementioned policy parameters, and then, use \eqref{eq:gs-est} to form $g_k$.

For applying the non-asymptotic bounds derived earlier for RSG-BGO, we require \ref{as:lipschitz} to hold. 
We verify this assumption for the special case of CVaR below.
Let $\text{C}_\alpha(K_{ x }(x^0)) $ denote the CVaR associated with a policy $\pi$ that is parameterized by $ x $. 
Then, using the likelihood ratio method, we arrive at the following variant of the policy gradient theorem under the CVaR objective (cf. \cite{tamar2015}):
\begin{align*}
	&\nabla_x C_\alpha(K_x(x^0))  = \mathbb{E}\bigg[ [ K_x(x^0)- V_\alpha(K_x(x^0)) ] \\
	& \quad \underbrace{\sum\limits_{m=0}^{\tau-1}  \nabla \log \pi_x(a_m |x_m) }_{(I)}  \bigg{|}  K_x(x^0) \geq V_{\alpha}(K_x(x^0)) \bigg].
\end{align*}
In the above, $K_x(x^0)$ and term (I) on the RHS are Lipschitz functions due to the policy gradient assumption mentioned above. In addition, if we assume that the distribution, say $F_x$, of $K_x(x^0)$ is a Lipschitz function in $x$, then we can infer that $\nabla_x C_\alpha(K_x(x^0))$ is sum of product of Lipschitz functions, implying assumption \ref{as:lipschitz}.
One could generalize this argument to the case when $\rho$ is a coherent risk measure, and the reader is referred to \cite{tamar2015} for details.

From the foregoing, it is apparent that the risk-PG oracle implemented using the gradient estimate \eqref{eq:gs-est} falls under the biased gradient oracle with an estimation error component scheme, i.e., either \ref{as:biased_sp_estimation_error}  or \ref{as:biased_gs_estimation_error}. For the risk-PG oracle to be of type \ref{as:biased_gs_estimation_error}, one would require an additional assumption that allows interchange of the expectation and differentiation operators, see Proposition 2.3 in \cite{asmussen2007stochastic} for an example.
Using the non-asymptotic bounds for RSG-BGO algorithm derived earlier, we can infer that the iteration complexity for the Risk-PG algorithm is:\\
(i) $ \mathcal{ O} (\frac{1}{\epsilon^3})$ under the oracle \ref{as:biased_sp_estimation_error} (using Theorem \ref{thm:biased_sp_esterr}); and 
(ii) $ \mathcal{ O} (\frac{1}{\epsilon^2})$ under the  oracle \ref{as:biased_gs_estimation_error} (using Theorem \ref{thm:biased_gs_esterr}).

\section{Convergence proofs}
\label{sec:proofs}

For notational convenience, we shall use $g_k \equiv g(x_k, \xi_k, \eta_k, m_k), \forall k \geq 1$.
Let $ \xi_{[k]}  := (\xi_1, \ldots, \xi_k) $, and $ \mathbb{E}_{\xi_{[k]}}  $ denote the expectation w.r.t. $ \xi_{[k]} $.

\subsection{Proofs for RSG-BGO algorithm with a non-convex objective} \label{pf:zrsg}

\subsubsection{Proof of Theorem \ref{thm:biased_sp_esterr}} \label{pf:biased_esterr}

In the proposition below, we state and prove a general result that holds for any choice of non-increasing stepsize sequence, perturbation constants and batch sizes. Subsequently, we specialize the result for the choice of parameters suggested in Theorem \ref{thm:biased_sp_esterr}, to prove the same. 

\begin{proposition} \label{prop:biased_esterr}
	Assume \ref{as:lipschitz} and \ref{as:boundedness}. With the oracle \ref{as:biased_sp_estimation_error}, suppose that the RSG-BGO algorithm is run with a non-increasing stepsize sequence satisfying $ 0 < \gamma_{ k } \le 1/L ,  \forall  k \ge 1 $ and with the probability mass function $ P_R(\cdot) $
	\begin{align} 
	P_R(k) := Prob\{R=k\} = \frac{\gamma_{ k }}{\sum _ { i = 1 } ^ { N} { \gamma _ { i} }}, \text{ } k = 1,\dots,N, \label{eq:prob}
	\end{align}
	then, for any $ N \ge 1 $, we have
	\begin{align}
	&\mathbb { E } \left[ \left\| \nabla f \left( x _ { R } \right) \right\| ^ { 2 } \right] 
	 \le \frac{1 }{\sum _ { k = 1 } ^ { N } \gamma_{ k } } \left[\frac { 2 D_f} { \left( 2-  { L }  \gamma _ { 1 } \right)} + \right.  \nonumber \\
	& \left.  2 B \sum _ { k = 1 } ^ { N} \mathcal{E}_k \left[ \frac{ \gamma _ { k} + L \gamma _ { k} ^ { 2 } }{ 2  -  { L }  \gamma _ { k}  }\right]      +  \sum_{ k = 1 }^{N} \frac{ L\gamma _ { k} ^ { 2 }}{\left( 2  -  { L }  \gamma _ { k} \right) } \left[ d\mathcal{E}_k ^2  + \frac{c_2}{\eta_k^2}\right] \right] , \label{eq:prop_biased_esterr_bound}
	\end{align}
	where $ \mathcal{E}_k = c_1 \eta_k^2  + \frac{c_3}{\eta_k \sqrt{m_k}}  $, constants $ c_1 $, $ c_2 $, $ c_3 $ are as defined in \ref{as:biased_sp_estimation_error}, $ B $ is as defined in \ref{as:boundedness}, and $ D_f $ is given in \eqref{eq:D_f}.	
\end{proposition}

\begin{proof}
	Let $ \Delta_k = g_k - \nabla f(x_k), \forall k \geq 1 $.
	We use the technique from \citep{ghadimi2013stochastic}. However, our proof involves significant deviations owing to the fact that the gradient estimates in \ref{as:biased_sp_estimation_error} have variance that scales inversely with perturbation constant $\eta_k$. Further, unlike \cite{ghadimi2013stochastic}, we have the batch size $m_k$ parameter that needs to be optimized.\\
	First, notice that
		\begin{align}
		\| x_{k+1} - x_{k} \| & = \|x _ { k } - \gamma _ { k } g_k - x _{ k }\|  =  \gamma_k  \| g_k \|, \nonumber
	\end{align}
and
\begin{align}
	& \mathbb { E }_{\xi_{ [k] }}  \left[ \Delta_k  \right]  = \mathbb { E }_{\xi_{ k }}  \left[ \Delta_k | \xi_{[k-1]}\right] = \mathbb { E }_{\xi_{ k }}  \left[ \Delta_k | x_k \right] \nonumber\\ 
	&= \mathbb { E }_{\xi_{ k }}  \left[ g_k - \nabla f(x_k) | x_k \right]  \leq \mathcal{E}_k \mathbf{1}_{d \times 1}, \label{eq:grad_bias}
\end{align}
where $ \mathcal{E}_k = c_1 \eta_k^2  + \frac{c_3}{\eta_k \sqrt{m_k}}  $, and
\begin{align}
&\mathbb { E }_{\xi_{[k]}} \left[ \left\| g_k \right\|^{2} \right] \le \left\| \mathbb { E }_{\xi_{[k]}}  \left[ g_k \right]  \right\| ^{2} + c_2 / \eta_k^2. \label{eq:grad_variance}
\end{align}
Under assumption \ref{as:lipschitz}, we have
	\begin{align}
	&f \left( x _ { k + 1 } \right) \nonumber \\
	& \leq f \left( x _ { k } \right) + \left\langle \nabla f \left( x _ { k } \right) , x _ { k + 1 } - x _ { k } \right\rangle + \frac { L } { 2 } \left\| x _ { k + 1 } - x _ { k } \right\| ^ { 2 } \nonumber \\ 
	& = f \left( x _ { k } \right) - \gamma _ { k} \left\langle \nabla f \left( x _ { k } \right) , g_k \right\rangle  + \frac { L } { 2 } \gamma _ { k} ^ { 2 } \left\| g_k\right\| ^ { 2 } \nonumber \\
	& = f \left( x _ { k } \right) - \gamma _ { k} \left\langle \nabla f \left( x _ { k } \right) , \nabla f(x_k) + \Delta_k \right\rangle \nonumber \\
	& \quad + \frac { L } { 2 } \gamma _ { k} ^ { 2 } \left\| g \left( x _ { k } , \xi _ { k }, m_k \right) \right\| ^ { 2 } \nonumber \\
	& = f \left( x _ { k } \right) - \gamma _ { k} \| \nabla f \left( x _ { k } \right) \|^2 - \gamma_k \left\langle \nabla f(x_k), \Delta_k \right\rangle \nonumber \\
	& \quad + \frac { L } { 2 } \gamma _ { k} ^ { 2 } \left\| g_k\right\| ^ { 2 }\label{eq:take_exp1} .
	\end{align}	
	Taking expectations with respect to $ \xi_{[k]} $ on both sides of \eqref{eq:take_exp1} and using \eqref{eq:grad_bias} and \eqref{eq:grad_variance}, we obtain
	\vspace{-.3mm}
	\begin{align}
	& \mathbb{E}_{\xi_{[k]}} \left[f \left( x _ { k + 1 } \right)\right] \nonumber \\
	& \leq \mathbb{E}_{\xi_{[k]}} \left[f \left( x _ { k } \right) \right] - \gamma _ { k} \mathbb{E}_{\xi_{[k]}} \left[\| \nabla f \left( x _ { k } \right) \|^2  \right] \nonumber\\
	& \quad - \gamma_k \mathbb{E}_{\xi_{[k]}} \left[ \left\langle \nabla f(x_k), \Delta_k \right\rangle  \right] \nonumber\\
	& \quad
	+ \frac { L } { 2 } \gamma _ { k} ^ { 2 } \left[ \left\| \mathbb { E }_{\xi_{[k]}}  \left[ g_k\right]  \right\| ^{2} + \frac{c_2}{\eta_k^2} \right] \nonumber \\
	&  \leq f \left( x _ { k } \right)
	- \gamma _ { k} \left\| \nabla f \left( x _ { k } \right) \right\| ^ { 2 } + \gamma _ { k } \mathcal{E}_k     \mathbb{E}_{\xi_{[k]}}  \| \nabla f \left( x _ { k } \right) \|_1  \nonumber \\
	& \quad + \frac { L } { 2 } \gamma _ { k} ^ { 2 } \left[ \left\| \nabla f \left( x _ { k } \right) \right\| ^ { 2 } + 2 \mathcal{E}_k  \mathbb{E}_{\xi_{[k]}}  \| \nabla f \left( x _ { k } \right) \|_1 \right. \nonumber\\
	& \left. \quad + {d}\mathcal{E}_k^2 + \frac{c_2}{\eta_k^2} \right] \label{eq:used_1norm_ineq}\\
	& \leq f \left( x _ { k } \right) - \left( \gamma _ { k} - \frac { L } { 2 } \gamma _ { k} ^ { 2 } \right)  \left\| \nabla f \left( x _ { k } \right) \right\| ^ { 2 } \nonumber\\
	& \qquad + \mathcal{E}_k B  \left( \gamma _ { k} + L \gamma _ { k} ^ { 2 } \right)  + \frac { L } { 2 } \gamma _ { k} ^ { 2 } \left[ d\mathcal{E}_k^2 + \frac{c_2}{\eta_k^2}\right], \nonumber 
	\end{align}
	where we have used the fact that $ - \|X\|_1 \leq \sum_{i=1}^{d} x_i $ for any vector $ X $ in arriving at the inequality \eqref{eq:used_1norm_ineq}. The last inequality follows from \ref{as:boundedness}.
	Re-arranging the terms, we obtain
	\begin{align*}
	&\gamma_{ k } \left\| \nabla f \left( x _ { k } \right) \right\| ^ { 2 }	
	\leq \frac{2}{\left( 2  -  { L }  \gamma _ { k} \right) } \bigg[ f \left( x _ { k } \right) -  \mathbb{E}_{\xi_{k}} f \left( x _ { k + 1 } \right) \bigg. \\
	& \quad \bigg.  + \mathcal{E}_k \left( \gamma _ { k} + L \gamma _ { k} ^ { 2 } \right)B \bigg] +  \frac{{ L } \gamma _ { k} ^ { 2 }}{\left( 2  -  { L }  \gamma _ { k} \right) } \left[ d\mathcal{E}_k ^2 + \frac{c_2}{\eta_k^2}\right].
	\end{align*}
	Now, summing up the inequality above over $ k = 1 $ to $ N $, and taking expectations with respect to the filtration generated by $\xi_{1}, \ldots, \xi_{N}$, we obtain
	\begin{align*}
	&	\sum _ { k = 1 } ^ { N} \gamma_{ k } \mathbb{E}_{\xi_{ [N]}}  \left\| \nabla f \left( x _ { k } \right) \right\| ^ { 2 } \\
	& \leq 2 \sum _ { k = 1 } ^ { N} \frac{ \left(\mathbb{E}_{\xi_{ [N]}}  f \left( x _ { k } \right) - \mathbb{E}_{\xi_{ [N]}}  f \left( x_{k+1} \right)  \right)}{\left( 2  -  { L }  \gamma _ { k} \right) } \\
	& \quad + 2 \sum _ { k = 1 } ^ { N} \mathcal{E}_k B \left( \frac{ \gamma _ { k} + L \gamma _ { k} ^ { 2 } }{ 2  -  { L }  \gamma _ { k}  }\right) \\
	& \quad +  \sum_{ k = 1 }^{N} \frac{ L \gamma _ { k} ^ { 2 }}{\left( 2  -  { L }  \gamma _ { k} \right) } \left[ d\mathcal{E}_k^2  + \frac{c_2}{\eta_k^2}\right]   \\
	& =  2 \left[\frac { f \left( x _ { 1 } \right) } {  \left( 2-  { L }  \gamma _ { 1 } \right) } - \sum _ { k = 2 } ^ { N } \left( \frac { \mathbb { E }_{\xi_{ [N] }}   f \left( x _ { k } \right) } { \left( 2-  { L }  \gamma _ { k-1 } \right) } - \frac { \mathbb { E }_{\xi_{ [N] }}  f \left( x _ { k } \right) } { \left( 2-  { L }  \gamma _ { k } \right) } \right) \right. \nonumber \\
	& \quad \left. - \frac { \mathbb { E }_{\xi_{ [N] }}  \left[ f \left( x _ { N + 1 } \right) \right] } { \left( 2-  { L }  \gamma _ { N } \right) } \right]+ 2 \sum _ { k = 1 } ^ { N} \mathcal{E}_k B \left( \frac{ \gamma _ { k} + L \gamma _ { k} ^ { 2 } }{ 2  -  { L }  \gamma _ { k}  }\right) \\
	& \quad +  \sum_{ k = 1 }^{N} \frac{ L \gamma _ { k} ^ { 2 }}{\left( 2  -  { L }  \gamma _ { k} \right) } \left[ d\mathcal{E}_k^2  + \frac{c_2}{\eta_k^2}\right].
	\end{align*}
	Noting that $ \mathbb { E }_{\xi_{ [N]}}  \left[ f \left( x _ { k } \right) \right] \ge f(x^*) $ and $ \left( \frac { 1 } { \left( 2-  { L }  \gamma _ { k-1 } \right) } - \frac { 1 } { \left( 2-  { L }  \gamma _ { k } \right) } \right)  \ge 0 $, we obtain
	\begin{align}
	& \sum _ { k = 1 } ^ { N } \gamma_{ k } \mathbb{E}_{\xi_{ [N] }} \left\| \nabla f \left( x _ { k } \right) \right\| ^ { 2 } 
	\nonumber\\
	& \leq \frac { 2 \left(f(x_1) - f(x^*) \right)} { \left( 2-  { L }  \gamma _ { 1 } \right)}  
	+ 2 \sum _ { k = 1 } ^ { N} \mathcal{E}_k B \left( \frac{ \gamma _ { k} + L \gamma _ { k} ^ { 2 } }{ 2  -  { L }  \gamma _ { k}  }\right) \nonumber \\
	& \quad + L \sum_{ k = 1 }^{N} \frac{ \gamma _ { k} ^ { 2 }}{\left( 2  -  { L }  \gamma _ { k} \right) } \left[ d\mathcal{E}_k^2  + \frac{c_2}{\eta_k^2}\right].\nonumber
	\end{align}
		The bound in \eqref{eq:prop_biased_esterr_bound} follows by using the distribution of R (specified in \eqref{eq:prob}) in the RHS above.
	
\end{proof}

We now specialize the result obtained in the proposition above, to derive the bounds in Theorem \ref{thm:biased_sp_esterr}.

\begin{proof} ~\textbf{\textit{(Theorem \ref{thm:biased_sp_esterr} (i))}}\\
	Recall that the stepsize $ \gamma_k \equiv \gamma $, perturbation constant $ \eta_k \equiv \eta$ and batch size $ m_k \equiv m $, $\forall k \geq 1$, where
		\begin{align} 
		\gamma =  min \bigg\{\frac{1}{L}, \frac{\gamma_0}{N^{2/3}}\bigg\}, \eta = \frac{\eta_0}{N^{1/6}}, m = m_0 N. \label{eq:biased_sp_par_local}
		\end{align}
		Let $\mathcal{E} =c_1 \eta^2  + \frac{c_3}{\eta \sqrt{m}}$. Then, combining \eqref{eq:prob} with \eqref{eq:prop_biased_esterr_bound}, we obtain
	\begin{align}
	&	\mathbb { E } \left[ \left\| \nabla f \left( x _ { R } \right) \right\| ^ { 2 } \right] \nonumber \\
	& \le  \frac{1}{{ N }\gamma } \left[ { 2 D_f} + 4N \gamma B \mathcal{E} +  { L N \gamma ^ { 2 }}\left[ d \mathcal{E}^2  + \frac{c_2}{\eta^2}\right] \right] \label{eq:using_1l}\\
	& \le \frac{ 2 D_f}{{ N }} max\bigg\{L, \frac{N^{2/3}}{\gamma_0}\bigg\}  \nonumber \\
	&  + 4 B \left(\frac{c_1 \eta_0^2 }{N^{1/3}}  + \frac{c_3 }{ \eta_0 \sqrt{m_0} N^{1/3} } \right)+   \frac{L\gamma_0}{N^{2/3}} \left[ \frac{dc_1^2 \eta_0^4 }{N^{2/3}} \right. \nonumber \\
	&   \left.+ \frac{2d c_1 c_3 \eta_0 }{ \sqrt{m_0} N^{2/3}} + \frac{d c_3^2}{ \eta_0^2 m_0 N^{2/3}}  + \frac{c_2}{ \eta_0^2  N^{-1/3}}\right] \label{eq:using_defs}.
	\end{align}
	In the above, inequality \eqref{eq:using_1l} follows by using the fact that $ \gamma \leq 1/L $, and the inequality \eqref{eq:using_defs} follows by using the definition of $\mathcal{E},  \gamma, \eta $ and $ m $. The bound in \eqref{eq:zrsg_o2_i} follows by rearranging terms.
\end{proof}
\begin{proof} ~\textbf{\textit{(Theorem \ref{thm:biased_sp_esterr} (ii))}}\\
	Recall the stepsize $ \gamma $, perturbation constant $ \eta $ from \eqref{eq:biased_sp_par_local}. Let $ m_k = m_0 k^\beta, \forall k \geq 1 $, for some constant $ \beta \in (0, 1) $ and $ m_0 > 0 $. 
	
	Combining \eqref{eq:prob} with \eqref{eq:prop_biased_esterr_bound}, we obtain
	\begin{align}
	&\mathbb { E } \left[ \left\| \nabla f \left( x _ { R } \right) \right\| ^ { 2 } \right] \nonumber \\
	& \le  \frac{1}{{ N }\gamma } \left[ { 2 D_f} + 4 N \gamma B c_1 \eta^2  + \frac{4  \gamma B c_3}{\eta \sqrt{m_0}} \sum _ { k = 1 } ^ { N}  k^{-\frac{\beta}{2}} \right. \nonumber\\
	& \quad \left. +  { L N \gamma ^ { 2 }}\left[ dc_1^2 \eta^4  + \frac{c_2}{\eta^2}\right] + \frac{ 2 L d\gamma ^ { 2 } c_1 c_3 \eta}{\sqrt{m_0}}\sum _ { k = 1 } ^ { N}  k^{-\frac{\beta}{2}} \right. \nonumber\\
	& \quad \left. + \frac{L d\gamma ^ { 2 } c_3^2}{\eta^2 m_0 } \sum _ { k = 1 } ^ { N} k^{-\beta}  \right] \label{eq:using_1l1} \\
	& \leq \frac{ 2 D_f}{{ N }\gamma}  + 4  B c_1 \eta^2  + \frac{4 B c_3}{N \eta \sqrt{m_0}} \left( \frac{N^{-\frac{\beta}{2} + 1}}{-\frac{\beta}{2} + 1} \right) \nonumber\\
	& \quad+   { L \gamma}\left[ dc_1^2 \eta^4  + \frac{c_2}{\eta^2}\right]   + \frac{ 2 L d \gamma c_1 c_3 \eta}{N \sqrt{m_0} } \left( \frac{N^{-\frac{\beta}{2} + 1}}{-\frac{\beta}{2} + 1} \right) \nonumber\\
	& \quad + \frac{L d\gamma c_3^2}{N \eta^2 m_0} \left(\frac{N^{-\beta+1}}{-\beta+1} \right) \nonumber \\
	& \le \frac{ 2 D_f}{{ N }} max\bigg\{L, \frac{N^{2/3}}{\gamma_0}\bigg\} + \frac{4B c_1 \eta_0^2 }{N^{1/3}}  \nonumber \\
	& \quad + \frac{4 B c_3 N^{1/6} }{ \eta_0 \sqrt{m_0}  N^{\frac{\beta}{2}}  \left( -\frac{\beta}{2} + 1 \right)} +   \frac{ L \gamma_0 }{N^{2/3}} \left[ \frac{dc_1^2 \eta_0^4 }{N^{2/3}} \right. \nonumber \\
	& \left. \quad + \frac{c_2 N^{1/3}}{\eta_0^2 } + \frac{ 2 d c_1 c_3 \eta_0 }{ \sqrt{m_0} N^{1/6} N^{\frac{\beta}{2}}  \left( -\frac{\beta}{2} + 1\right)}  \right. \nonumber \\
	& \left. \quad + \frac{ dc_3^2 N^{1/3}}{ \eta_0^2 m_0 N^{\beta}  \left(-\beta+1 \right)}  \right] \label{eq:using_defs1} 
	\end{align}
	In the above, inequality \eqref{eq:using_1l1} follows by using the fact that $ \gamma \leq 1/L $ and $ m_k = k^\beta $. The inequality \eqref{eq:using_defs1} follows by using the definition of $ \gamma$ and $ \eta $. The bound in \eqref{eq:zrsg_o2_ii} follows by rearranging terms.
\end{proof}

\subsubsection{Proof of Theorem \ref{thm:biased_gs_esterr} } \label{pf:thm_biased_gs_esterr}
\begin{proof}
	Following the proof of Proposition \ref{prop:biased_esterr}, we obtain
	\begin{align}
		&	\mathbb { E } \left[ \left\| \nabla f \left( x _ { R } \right) \right\| ^ { 2 } \right] 
		\le \frac{1 }{\sum _ { k = 1 } ^ { N } \gamma_{ k } } \left[\frac { 2 D_f} { \left( 2-  { L }  \gamma _ { 1 } \right)} \right. \nonumber \\
		& \left. \quad + 2 \sum _ { k = 1 } ^ { N} \mathcal{E}_k \left( \frac{ \gamma _ { k} + L \gamma _ { k} ^ { 2 } }{ 2  -  { L }  \gamma _ { k}  }\right) \mathbb { E }_{\xi_{ [N]}}  \| \nabla f \left( x _ { k } \right) \|_1 \right. \nonumber \\
		& \left. \quad  + L \sum_{ k = 1 }^{N} \frac{ \gamma _ { k} ^ { 2 }}{\left( 2  -  { L }  \gamma _ { k} \right) } \left[ d\mathcal{E}_k^2  + {c_2}{\eta_k^2} + \tilde{c_2}\right] \right], \nonumber
	\end{align}
		where $ \mathcal{E}_k = c_1 \eta_k  + \frac{c_3}{\eta_k \sqrt{m_k}}  $. Using arguments similar to those employed in Theorem \ref{thm:biased_sp_esterr}, we obtain
	\begin{align*} 
		&	\mathbb { E } \left\| \nabla f \left( x _ { R } \right) \right\| ^ { 2 } 
		\le  \frac{ 2 D_f}{{ N }\gamma} + 4 B \mathcal{E}  +   { L \gamma}\left[ d\mathcal{E}^2  + {c_2}{\eta^2} + \tilde{c_2}\right].
	\end{align*}
	The main claim follows by plugging values of $\gamma$, $ \eta$, and $ m $,  as defined in Theorem  \ref{thm:biased_gs_esterr} in the above equation.
\end{proof}

\subsection{Proofs for RSG-BGO algorithm with convex objective} \label{pf:cvx_zrsg}

\subsubsection{Proof of Theorem \ref{thm:biased_sp_esterr_convex}} \label{pf:thm_biased_sp_esterr_convex}

In the proposition below, we state and prove a general result that holds for any choice of non-increasing stepsize sequence, perturbation constants and batch sizes. Subsequently, we specialize the result for the choice of parameters suggested in Theorem \ref{thm:biased_sp_esterr_convex}, to prove the same. 

\begin{proposition} \label{prop:biased_esterr_convex}
	Assume \ref{as:lipschitz} and \ref{as:boundedness}. With the oracle \ref{as:biased_sp_estimation_error}, suppose that the ZRSG algorithm is run with a non-increasing stepsize sequence satisfying $ 0 < \gamma_{ k } \le 1/L ,  \forall  k \ge 1 $ and with the probability mass function $ P_R(\cdot) $  as defined in \eqref{eq:prob},
	then, for any $ N \ge 1 $, we have
	\begin{align} 
		& 	\mathbb { E } \left[ f \left( x _ { R } \right) \right] - f (x^ { * }) \nonumber \\
		&\le \frac{1}{\sum _ { k = 1 } ^ { N }\gamma _ { k} } \bigg[ \frac { D^2} { \left( 2-  { L }  \gamma _ { 1 } \right)} + 2\sqrt{d} D \sum _ { k = 1 } ^ { N } \mathcal{E}_k  \frac{(\gamma_k + L \gamma_k^2) }{(2 - L \gamma_k)} \nonumber\\
		& \qquad + \sum _ { k = 1 } ^ { N } \frac{\gamma _ { k } ^ { 2 } }{(2-L\gamma_k)}\bigg( d \mathcal{E}_k^2 + \frac{c_2}{\eta_k^2}\bigg) \bigg], \label{eq:prop_biased_esterr_convex_bound}
	\end{align}
	where $ \mathcal{E}_k = c_1 \eta_k^2  + \frac{c_3}{\eta_k \sqrt{m_k}}  $, constants $ c_1 $, $ c_2 $ and $ c_3 $ are as defined in \ref{as:biased_sp_estimation_error}, and $ D $ as defined in \ref{eq:dia}.	
\end{proposition}

\begin{proof}
	Let $ \Delta_k = g_k - \nabla f(x_k) $ and  $\omega _ { k } = \left\| x _ { k } - x ^ { * } \right\|, \forall k \geq 1$ . Then for
	any $ k = 1,\dots,N $, we have,
	\begin{align}
	\omega _ { k + 1 } ^ { 2 } & = \| x_{k+1} - x ^ { * } \|^2 \nonumber \\
		& = \|x_{k} - \gamma_k g_k - x^* \|^2 \nonumber\\
	& = \omega _ { k } ^ { 2 } - 2 \gamma _ { k } \left\langle g_k , x _ { k } - x ^ { * } \right\rangle + \gamma _ { k } ^ { 2 } \left\| g_k \right\| ^ { 2 }. \label{eq:take_exp3}
	\end{align}

Taking expectations with respect to $ \xi_{[k]} $ on both sides of \eqref{eq:take_exp3}, and using \eqref{eq:grad_bias}, \eqref{eq:grad_variance}, we obtain
	\begin{align*}
	&\mathbb { E }[\omega _ { k + 1 } ^ { 2 }] \\
	& \leq
	\E [\omega _ { k } ^ { 2 }] - 2 \gamma _ { k } \left\langle \nabla f(x_k) , x _ { k } - x^* \right\rangle  \\
	& \text{ } - 2 \gamma _ { k } \E \left[\left\langle\Delta_k ,  x _ { k } - x^* \right\rangle\right] + \gamma _ { k } ^ { 2 } \bigg[ \left\| \mathbb { E }_{\xi_{[k]}}  \left[ g_k\right]  \right\| ^{2} + \frac{c_2}{\eta_k^2} \bigg]\\
	& \leq 	\E [\omega _ { k } ^ { 2 }]  - 2 \gamma _ { k } \left\langle \nabla f \left( x _ { k } \right) , x _ { k } - x^* \right\rangle \\
	& \text{ } + 2 \gamma_k \mathcal{E}_k\| x _ { k } - x^*  \|_1 \\
	& \text{ } + \gamma _ { k } ^ { 2 } \bigg[ \| \nabla f \left( x _ { k } \right) \| ^ { 2 } + 2 \sqrt{d} \mathcal{E}_k \| \nabla f \left( x _ { k } \right) \| + d\mathcal{E}_k ^2 + \frac{c_2}{\eta_k^2} \bigg].
\end{align*}	
where the last inequality follows from the fact that $ - \sum_{i=1}^{d} x_i \leq \|X\|_1 $ for any vector $ X $. Now, using the fact that $ f(\cdot) $ is convex, we have 
	$\left\| \nabla f \left( x _ { k } \right) \right\| ^ { 2 } \leq L \left\langle \nabla f \left( x _ { k } \right) , x _ { k } - x ^ { * } \right\rangle$,
	further
	from \ref{as:lipschitz}, we have $ \| \nabla f(x_k) \| \leq L \| x_k - x^* \|  $.
	Plugging it in equation above, we obtain,
\begin{align*}
	&\mathbb { E } [\omega _ { k + 1 } ^ { 2 }] \\
	& \leq \E [\omega _ { k } ^ { 2 }] - 2 \gamma _ { k } \left\langle \nabla f \left( x _ { k } \right) , x _ { k } - x^* \right\rangle +  2 \gamma_k \mathcal{E}_k\| x _ { k } - x^*  \|_1 \\
	& \quad + \gamma _ { k } ^ { 2 } \bigg[ L \left\langle \nabla f \left( x _ { k } \right) , x _ { k } - x ^ { * } \right\rangle   + 2\sqrt{d} \mathcal{E}_k L \| x_k - x^* \| \\
	& \quad + d\mathcal{E}_k^2+ \frac{c_2}{\eta_k^2} \bigg]\\ 
	& \leq \E [\omega _ { k } ^ { 2 }]  - (2\gamma _ { k } - L \gamma_k^2)\left[ f \left( x _ { k } \right) - f (x^* ) \right] \\
	& \quad + 2 \sqrt{d}   \omega_k \mathcal{E}_k (\gamma_k + L \gamma_k^2) + \gamma _ { k } ^ { 2 } \bigg[  d\mathcal{E}_k^2+ \frac{c_2}{\eta_k^2} \bigg],
\end{align*}	
where the last inequality follows from the fact that $ f(\cdot) $ is convex along with $ \| X \|_1 \leq \sqrt{d} \| X \| $ for any vector $ X $. Re-arranging the terms, we obtain
	\begin{align*}
	&  \gamma _ { k } \left[ f \left( x _ { k } \right) - f (x^*) \right]  \leq  \frac{1}{(2 - L \gamma_k)} \bigg[ \omega _ { k } ^ { 2 }  - \mathbb { E } [\omega _ { k +1 } ^ { 2 }] \\
	 & \qquad + 2 \sqrt{d}  \omega_{ k }  \mathcal{E}_k (\gamma_k + L \gamma_k^2)  + \gamma _ { k } ^ { 2 } \bigg(d \mathcal{E}_k^2+ \frac{c_2}{\eta_k^2} \bigg) \bigg].
	\end{align*}
	Now summing up the inequality above from $ k = 1 $ to $ N $ and taking expectation on both sides of above equation, we obtain
	\begin{align*}
		&\sum _ { k = 1 } ^ { N } \gamma_k \mathbb{E}_{\xi_{ [N]}}\left[ f \left( x _ { k } \right) - f (x^ { * }) \right] \\
		& \leq
		\sum _ { k = 1 } ^ { N } \frac{\mathbb{E}_{\xi_{ [N]}} [\omega _ { k } ^ { 2 }] - \mathbb{E}_{\xi_{ [N]}}[\omega _ { k + 1 } ^ { 2 }]}{(2 - L \gamma_k)}  \\
		&  + 2 \sqrt{d}   \sum _ { k = 1 } ^ { N }\mathbb { E }_{\xi_{ [N] }}  \left[ \omega _ { k }  \right]  \mathcal{E}_k   \frac{(\gamma_k + L \gamma_k^2) }{(2 - L \gamma_k)} \\
		& + \sum _ { k = 1 } ^ { N } \frac{\gamma _ { k } ^ { 2 } }{(2-L\gamma_k)}\bigg( d \mathcal{E}_k^2 + \frac{c_2}{\eta_k^2}\bigg).
	\end{align*}
	\noindent
	Using the fact that $ \mathbb { E }_{\xi_{ [N] }}  \left[ \omega _ { k }  \right] \geq 0 $ and $ L \gamma_k \leq 1 $ for all $ k \geq 1 $, we obtain
	\begin{align*}
		& \sum _ { k = 1 } ^ { N } \gamma_k \mathbb{E}_{\xi_{ [N]}}\left[ f \left( x _ { k } \right) - f (x^ { * }) \right]   \\
		& =  \bigg[\frac { \omega _ { 1 } ^ { 2 } } {  \left( 2-  { L }  \gamma _ { 1 } \right) } - \frac { \mathbb { E }_{\xi_{ [N] }}  \left[ \omega _ { N+1 } ^ { 2 } \right] } { \left( 2-  { L }  \gamma _ { N } \right) } \\
		& \quad - \sum _ { k = 2 } ^ { N } \left( \frac { 1 } { \left( 2-  { L }  \gamma _ { k-1 } \right) } - \frac { 1 } { \left( 2-  { L }  \gamma _ { k } \right) } \right) \mathbb { E }_{\xi_{ [N] }}  \left[ \omega _ { k } ^ { 2 } \right]  \bigg]\\
		& \quad + 2 \sqrt{d}  \sum _ { k = 1 } ^ { N }   \mathbb { E }_{\xi_{ [N] }}  \left[ \omega _ { k }  \right] \mathcal{E}_k \frac{(\gamma_k + L \gamma_k^2) }{(2 - L \gamma_k)} \\
		& \quad + \sum _ { k = 1 } ^ { N } \frac{\gamma _ { k } ^ { 2 } }{(2-L\gamma_k)}\bigg(d \mathcal{E}_k^2 + \frac{c_2}{\eta_k^2}\bigg)\\
		& \leq \frac {  D^2} { \left( 2-  { L }  \gamma _ { 1 } \right)} + 2 \sqrt{d }D \sum _ { k = 1 } ^ { N } \mathcal{E}_k \frac{(\gamma_k + L \gamma_k^2) }{(2 - L \gamma_k)} \\
		& \quad + \sum _ { k = 1 } ^ { N } \frac{\gamma _ { k } ^ { 2 } }{(2-L\gamma_k)}\bigg( d \mathcal{E}_k^2 + \frac{c_2}{\eta_k^2}\bigg)
	\end{align*}
where the last inequality follows by using \ref{eq:dia}, i.e., $  \mathbb { E }_{\xi_{ [N] }}  \left[ \omega _ { k }  \right] \leq D $. We conclude by combining the above result with \eqref{eq:prob}.
\end{proof}

\begin{proof} ~\textbf{(Theorem \ref{thm:biased_sp_esterr_convex})}\\
	Recall that the stepsize $ \gamma_k \equiv \gamma $, perturbation constant $ \eta_k \equiv \eta$ and batch size $ m_k \equiv m $, $\forall k \geq 1$ is as defined in \eqref{eq:biased_sp_par_local}.
	Combining \eqref{eq:prob} with \eqref{eq:prop_biased_esterr_convex_bound}, we obtain
	\begin{align}
		&\mathbb { E } \left[ f \left( x _ { R } \right) \right] - f (x^ { * }) \nonumber \\
		& \le \frac{1}{\sum _ { k = 1 } ^ { N }\gamma _ { k} } \bigg[ \frac { D^2} { \left( 2-  { L }  \gamma _ { 1 } \right)} + 2 \sqrt{d} D \sum _ { k = 1 } ^ { N } \mathcal{E}_k \frac{(\gamma_k + L \gamma_k^2) }{(2 - L \gamma_k)} \nonumber\\
		& \qquad + \sum _ { k = 1 } ^ { N } \frac{\gamma _ { k } ^ { 2 } }{(2-L\gamma_k)}\bigg(d \mathcal{E}_k^2 + \frac{c_2}{\eta_k^2}\bigg) \bigg] \nonumber\\
		& \le  \frac{1}{{ N }\gamma } \left[ { D^2} + 4 \sqrt{d} D N \gamma \mathcal{E} + N \gamma^2\bigg(d \mathcal{E}^2 + \frac{c_2}{\eta^2}\bigg) \right] \label{eq:using_1l_cvx}\\
		& = \frac{ D^2}{{ N }\gamma}   + 4 \sqrt{d} D \left(c_1 \eta^2  + \frac{c_3}{\eta \sqrt{m}} \right) \nonumber \\
		& \quad + \gamma \left[ dc_1^2 \eta^4 + 2 dc_1 c_3 \frac{\eta}{\sqrt{m}} + \frac{dc_3^2}{\eta^2 {m}}  + \frac{c_2}{\eta^2} \right]\nonumber\\
		&  \leq  \frac{  L D^2}{{ N }} + \frac{1}{N^{1/3}}\bigg[{ \frac{D^2}{\gamma_0} } + 4 \sqrt{d} D \left({c_1 \eta_0^2 }  + \frac{c_3}{\eta_0 \sqrt{m_0}} \right) \nonumber  \\
		& + \gamma_0 \bigg(  \frac{d c_1^2 \eta_0^4  }{N} +  \frac{2  d c_1 c_3 \eta_0 }{ \sqrt{m_0 } N} + \frac{dc_3^2 }{\eta_0^2 m_0 N}  + \frac{c_2}{\eta_0^2} \bigg) \bigg]. \label{eq:using_defs_cvx}  
	\end{align}		
	In the above, inequality \eqref{eq:using_1l_cvx} follows by using the fact that $ \gamma \leq 1/L $, and the inequality \eqref{eq:using_defs_cvx} follows by using the definition of $ \gamma, \eta $ and $ m $.
\end{proof}

\subsubsection{Proof of Theorem \ref{thm:biased_gs_esterr_convex}} \label{pf:thm_biased_gs_esterr_convex}
\begin{proof}
	Following the initial passage in the proof of Theorem \ref{thm:biased_sp_esterr_convex} with  \ref{as:biased_sp_estimation_error}, we obtain the following inequality for the case of \ref{as:biased_gs_estimation_error} considered here:
	\begin{align*}
		& \mathbb { E } \left[ f \left( x _ { R } \right) \right] - f (x^ { * }) \\
		& \le \frac{1}{\sum _ { k = 1 } ^ { N }\gamma _ { k} } \left[ \frac {D^2} { \left( 2-  { L }  \gamma _ { 1 } \right)} + 2\sqrt{d} D \sum _ { k = 1 } ^ { N}  \mathcal{E}_k \left( \frac{ \gamma _ { k} + L \gamma _ { k} ^ { 2 } }{ 2  -  { L }  \gamma _ { k}  }\right) \right. \\
		& \qquad \left.  +  \sum_{ k = 1 }^{N} \frac{ \gamma _ { k} ^ { 2 }}{\left( 2  -  { L }  \gamma _ { k} \right) } \left[ d\mathcal{E}_k^2 + {c_2}{\eta_k^2} + \tilde{c_2}\right] \right],
	\end{align*}
	where $ \mathcal{E}_k = c_1 \eta_k  + \frac{c_3}{\eta_k \sqrt{m_k}}  $. Now, using the choice of parameters in \eqref{eq:biased_gs_par}, followed by simplifications similar to those in the proof of Theorem \ref{thm:biased_sp_esterr_convex}, we obtain
	\begin{align} 
		&		\mathbb { E } \left[ f \left( x _ { R } \right) \right] - f (x^ { * }) \nonumber\\
		& \leq \frac{ D^2}{{ N }\gamma}   + 4 \sqrt{d} D \mathcal{E} +   {  \gamma}\left[ d \mathcal{E}^2  + {c_2}{\eta^2} + \tilde{c_2}\right] \nonumber\\
		& \le  \frac{  L D^2}{{ N }}   + \frac{1}{\sqrt{N}} \bigg[ \frac{ D^2}{\gamma_0} + 4 \sqrt{d}D \left({c_1 }{\eta_0} + \frac{c_3}{\eta_0 \sqrt{m_0}} \right) \nonumber \\
		& \quad +  \gamma_0 \left( \frac{d c_1^2 \eta_0^2  }{N} +  \frac{2  d c_1 c_3 }{\sqrt{m_0}N} + \frac{ {d} c_3^2 }{ \eta_0^2 m_0 {N}} + \frac{  c_2  \eta_0^2 }{N} + {\tilde{c_2}} \right)  \bigg].\nonumber
	\end{align}
\end{proof}

\subsection{Proofs for SGD-BGO algorithm} \label{pf:zsgd}


\subsubsection{Proof of Theorem~\ref{thm:convex_sp_esterr}} \label{pf:biased_convex_sp_esterr}

The proof proceeds through a sequence of lemmas. We follow the technique from \cite{jain2019making}, and our proof involves significant deviations owing to the fact that unbiased gradient information is not available, leading to additional terms involving perturbation constants (arising out of gradient bias), and batch sizes (arising due to estimation errors).

Recall that $ N_i, l $ is defined as follows:
\begin{align} 
\text{Let  }   l & = \inf\{i : N\cdot 2^{-i} \leq 1\}, \nonumber\\
 N_i & =  N - \lceil N\cdot 2^{-i}\rceil,\ 0\leq i\leq l,\mbox{ and } N_{l+1} := N.\label{eq:def_Ni_local}
\end{align}  
Further, when $ N_i < k \leq N_{i+1}, 0\leq i \leq l $, stepsize $ \gamma_k$, perturbation constant $\eta_k $, and batch size $ m_k $  is defined as follows:
\begin{align}
\gamma_k = \frac{\gamma_0 \cdot 2^{-i}}{N^{2/3}},  \text{ } \eta_k = \frac{ \eta_0 2^{-i/4}}{N^{1/6}}  \textrm{ and } m_k = 2^i N  ,
\label{eq:weak_step_size_sp_esterr_local}
\end{align} 
for some constant $ \gamma_0, \eta_0  >0 $. Note that, unlike \cite{jain2019making}, parameters $ \eta_k $ and $ m_k $ are local to our setting, and due to the inverse scaling of variance in gradient estimates with $ \eta_k $, the stepsizes $ \gamma_k$ chosen is of $\O(\frac{1}{N^{2/3}}) $ and not $ \O(\frac{1}{\sqrt{N}}) $. 

We divide the proof into phases $ N_i $, let $ x_1,\ldots,x_N $ be the output of the SGD-BGO algorithm. We start with a variant of Lemma 1 from \cite{jain2019making}. In comparison to their result, our claim below features additional factors involving perturbation constant $\eta_k $ and batch size $ m_k $ owing to the zeroth-order setting we consider.

\begin{lemma} \label{lm:biased_convex_sp_esterr}
Assume \ref{as:boundedness}. With the oracle \ref{as:biased_sp_estimation_error}, suppose that the SGD-BGO algorithm is run with stepsize sequence $ \{\gamma_{ k }\}_{k=1}^{N} $. Then, given any $ 1 < k_0 < k_1 \leq N $, we have
	\begin{align*}
	\sum _{ k = k_0 } ^ { k_1 } 2 \gamma_k \mathbb{E} \left[ f \left( x _ { k } \right) - f (x_{k_0}) \right]  \leq 
	\sum _ { k = k_0 } ^ { k_1 } 2 \sqrt{d} \gamma_k  D\mathcal{E}_k + \gamma _ { k } ^ { 2 } \mathcal{B}_k^2,
	\end{align*}
	where $ \mathcal{B}_k^2 = \bigg[ B^2  + 2\sqrt{d} B \mathcal{E}_k  + d\mathcal{E}_k ^2+ \frac{c_2}{\eta_k^2} \bigg] $, $ \mathcal{E}_k = c_1 \eta_k^2  + \frac{c_3}{\eta_k \sqrt{m_k}}  $, constants $ c_1,c_2 $ is as defined in \ref{as:biased_sp_estimation_error} and $ D $ is as defined in \eqref{eq:dia}.	
\end{lemma}

\begin{proof}
	Let $ \Delta_k = g_k - \nabla f(x_k) $ and $\omega _ { k } = \left\| x _ { k } - x_{ k_0 } \right\|,  \forall k \geq 1$. Then for
	any $ k = 1,\dots,N $, we have
	\begin{align}
	\omega _ { k + 1 } ^ { 2 } & = \| x_{k+1} - x_{k_0} \|^2 \nonumber \\
	& =  \left\| x _ { k } - \gamma _ { k } g_k - x _{ k_0 }\right\| ^ { 2 } \nonumber \\
	& = \omega _ { k } ^ { 2 } - 2 \gamma _ { k } \left\langle g_k , x _ { k } - x _{ k_0 } \right\rangle  + \gamma _ { k } ^ { 2 } \left\| g_k \right\| ^ { 2 } \nonumber \\
	& = \omega _ { k } ^ { 2 } - 2 \gamma _ { k } \left\langle \nabla f(x_k)  + \Delta_k, x _ { k } - x _{ k_0 } \right\rangle  + \gamma _ { k } ^ { 2 } \left\| g_k \right\| ^ { 2 } \nonumber \\
	& = \omega _ { k } ^ { 2 } - 2 \gamma _ { k } \left\langle \nabla f(x_k) , x _ { k } - x _{ k_0 } \right\rangle \nonumber \\
	& \quad - 2 \gamma _ { k } \left\langle \Delta_k, x _ { k } - x _{ k_0 } \right\rangle+ \gamma _ { k } ^ { 2 } \left\| g_k \right\| ^ { 2 }
	. \label{eq:take_exp}
	\end{align}
	Taking expectations with respect to $ \xi_{[k]} $ on both sides of \eqref{eq:take_exp}, and using \eqref{eq:grad_bias}, \eqref{eq:grad_variance}, we obtain
	\begin{align*}
	&\mathbb { E }[\omega _ { k + 1 } ^ { 2 }] \\
	& \leq
	\E [\omega _ { k } ^ { 2 }] - 2 \gamma _ { k } \left\langle \nabla f(x_k) , x _ { k } - x _{ k_0 } \right\rangle  \\
	& \text{ } - 2 \gamma _ { k } \E \left[\left\langle\Delta_k ,  x _ { k } - x _{ k_0 }\right\rangle\right] + \gamma _ { k } ^ { 2 } \bigg[ \left\| \mathbb { E }_{\xi_{[k]}}  \left[ g_k\right]  \right\| ^{2} + \frac{c_2}{\eta_k^2} \bigg]\\
	& \leq 	\E [\omega _ { k } ^ { 2 }]  - 2 \gamma _ { k } \left\langle \nabla f \left( x _ { k } \right) , x _ { k } - x _{ k_0 }\right\rangle \\
	& \text{ } + 2 \gamma_k \mathcal{E}_k\| x _ { k } - x_{k_0}  \|_1 \\
	& \text{ } + \gamma _ { k } ^ { 2 } \bigg[ \| \nabla f \left( x _ { k } \right) \| ^ { 2 } + 2 \sqrt{d} \mathcal{E}_k \| \nabla f \left( x _ { k } \right) \| + d\mathcal{E}_k ^2 + \frac{c_2}{\eta_k^2} \bigg].
	\end{align*}	
	where the last inequality follows from the fact that $ - \sum_{i=1}^{d} x_i \leq \|X\|_1 $ for any vector $ X $. Now, using \ref{as:boundedness}, i.e., $ \| \nabla f(x) \| \leq  \| \nabla f(x) \| _1\leq B $, we obtain
	\begin{align*}
	&\mathbb { E } [\omega _ { k + 1 } ^ { 2 }] \\
	& \leq \E [\omega _ { k } ^ { 2 }] - 2 \gamma _ { k } \left\langle \nabla f \left( x _ { k } \right) , x _ { k } - x _{ k_0 } \right\rangle \\
	& \quad + 2 \gamma_k \mathcal{E}_k\| x _ { k } - x_{k_0}  \|_1 \\
	& \quad + \gamma _ { k } ^ { 2 } \bigg[ B^2  + 2\sqrt{d} B \mathcal{E}_k + d\mathcal{E}_k^2+ \frac{c_2}{\eta_k^2} \bigg]\\ 
	& \leq \E [\omega _ { k } ^ { 2 }]  - 2 \gamma _ { k } \left[ f \left( x _ { k } \right) - f (x_{ k_0 }) \right] + 2 \sqrt{d} \gamma_k  \omega_k \mathcal{E}_k\\
	& \quad+ \gamma _ { k } ^ { 2 } \bigg[ B^2  + 2\sqrt{d} B \mathcal{E}_k + d\mathcal{E}_k^2+ \frac{c_2}{\eta_k^2} \bigg],
	\end{align*}	
	where the last inequality follows from the fact that $ f(\cdot) $ is convex along with $ \| X \|_1 \leq \sqrt{d} \| X \| $ for any vector $ X $. Re-arranging the terms, we obtain
	\begin{align*}
	& 2 \gamma _ { k } \left[ f \left( x _ { k } \right) - f (x_{ k_0 }) \right]	
	\leq \E [\omega _ { k } ^ { 2 }]  - \mathbb { E } [\omega _ { k +1 } ^ { 2 }] \\
	& \quad+ 2\sqrt{d} \gamma_k  \omega_k \mathcal{E}_k + \gamma _ { k } ^ { 2 } \bigg[ B^2  + 2\sqrt{d} B \mathcal{E}_k + d\mathcal{E}_k^2+ \frac{c_2}{\eta_k^2} \bigg].
	\end{align*}
	We conclude by  summing the above equation over $ k = k_0 $ to $ k_1 $, taking expectations, and using \eqref{eq:dia}, i.e., $\left\|x_{1}-x^{*}\right\| \leq D$.
\end{proof}

\begin{lemma} \label{lm:convex_sp_std_analysis_esterr}
Under conditions of Lemma \ref{lm:biased_convex_sp_esterr}, with $ \gamma_k = \gamma, \eta_k = \eta, \forall k \geq 1 $, for any $ N \geq 1$, we have	
	\[
	\sum_{k=1}^{N} \E \left[f\left(x_{k}\right)-f\left(x^{*}\right)\right] \leq \frac{D^{2}}{2 \gamma}  +2 N D \sqrt{d}  \mathcal{E} + \frac{ N \gamma B^2}{2}, 
	\]	
	where $ \mathcal{E} = c_1\eta^2 + \frac{c_3}{\eta \sqrt{m}} $, $ c_1 $ is as defined in \ref{as:biased_sp_estimation_error}, $ B $ is as defined in \ref{as:boundedness} and $ D $ is as defined in \eqref{eq:dia}.
\end{lemma}

\begin{proof} 
	Let $ \Delta_k =  g_k - \nabla f(x_k)$, then we have $ x_{k+1}  = x_k - \gamma_k \left(\nabla f(x_k) + \Delta_k  \right) $. 
	Using the definition of convexity, we obtain
	\begin{align} 
	&f\left(x_{k}\right)-f\left(x^{*}\right) \nonumber \\
	 & \leq \nabla f(x_k)^{\top}\left(x_{k}-x^{*}\right) \nonumber \\
	&=\left(\frac{x_{k}-x_{k+1}}{\gamma_k} - \Delta_k\right)^{\top}\left(x_{k}-x^{*}\right) \nonumber \\
	& = \frac{1}{\gamma_k}\left({x_{k}-x_{k+1}} - \gamma_k \Delta_k\right)^{\top}\left(x_{k}-x^{*}\right) 
	\nonumber\\ &=\frac{1}{2 \gamma_k}\left(\left\|x_{k}-x^{*}\right\|^{2}+\left\|x_{k}-x_{k+1} - \gamma_k \Delta_k\right\|^{2} \right. \nonumber \\
	& \quad \left. -\left\|x_{k+1}-x^{*} + \gamma_k \Delta_k\right\|^{2}\right) \label{eq:used_elementary_iden}\\ &=\frac{1}{2 \gamma_k}\left(\left\|x_{k}-x^{*}\right\|^{2}-\left\|x_{k+1}-x^{*}  + \gamma_k \Delta_k \right\|^{2}\right) \nonumber \\
	& \quad+\frac{\gamma_k}{2}\left\|\nabla f(x_k)\right\|^{2}, \nonumber
	\end{align}
	where we have used the identity $2 a^{\top} b=\|a\|^{2}+\|b\|^{2}-\|a-b\|^{2}$ in arriving at the equality in \eqref{eq:used_elementary_iden}. Using $\left\|\nabla f(x_k)\right\|^2 \leq B^2$, we have
	\begin{align*}
	& f\left(x_{k}\right)-f\left(x^{*}\right) 
	\leq \frac{1}{2 \gamma_k}\left(\left\|x_{k}-x^{*}\right\|^{2}-\left\|x_{k+1}-x^{*}  \right\|^{2}  \right. \\
	& \quad \left. - 2 \gamma_k (x_{k+1}-x^{*})^T\Delta_k  \right)+\frac{\gamma_k B^2}{2}.
	\end{align*}
	Taking expectation, we obtain
	\begin{align} 
	&\E [f\left(x_{k}\right)-f\left(x^{*}\right)] \nonumber \\
	& \leq \frac{1}{2 \gamma_k}\left( \E [\left\|x_{k}-x^{*}\right\|^{2}]- \E [\left\|x_{k+1}-x^{*}  \right\|^{2}] \right. \nonumber \\
	& \quad \left. + 2 \gamma_k \E \bigg[ \| (x_{k+1}-x^{*})\|_1 \| \Delta_k\|_{\infty}   \bigg] \right)  +\frac{\gamma_k B^2}{2} \nonumber\\
	& \leq \frac{1}{2 \gamma_k}\left( \E [\left\|x_{k}-x^{*}\right\|^{2}]- \E [\left\|x_{k+1}-x^{*}  \right\|^{2}]  \right. \nonumber \\
	& \quad \left. + 2 \gamma_k \sqrt{d} \mathcal{E}_k \E [\|x_{k+1}-x^{*}\| ] \right)+\frac{\gamma_k B^2}{2}, \label{eq:summing}
	\end{align}
	where the last inequality follows from the fact that $ - \sum_{i=1}^{d} x_i \leq \|X\|_1 \leq \sqrt{d} \| X \| $ for any vector $ X $. We conclude by summing \eqref{eq:summing} over $k,$ with $ \gamma_k = \gamma, \eta_k = \eta, \forall k \geq 1 $, and using the fact that  $\left\|x_{1}-x^{*}\right\| \leq D$.
\end{proof}

\begin{proof}  ~\textbf{(Theorem~\ref{thm:convex_sp_esterr})}\\	
	Recall the definition of $ N_i, l $ from equation \eqref{eq:def_Ni_local} and let $n_i$, $0\leq i\leq l+1$ be defined as follows:
	\begin{align}\label{eq:def_kau}n_i = \arg\inf_{ N_i <k \leq N_{i+1}} \mathbb{E}[f(x_k)],\ i\in [l+1], \nonumber \\  \mbox{ and }n_0 = \arg\inf_{\lceil \frac{N}{4}\rceil \leq k\leq N_1 }\mathbb{E}[f(x_k)].\end{align} 
	We split the horizon $N$ into $l$ phases, then to show that the function value for the final iterate $ x_N $ in the last phase ($N_{l+1} = N$) is close to optima $ f(x^*) $. Using the fact that $n_{l+1}= N$, we have
	\begin{align}
	& \mathbb{E}[f(x_{N})] = \mathbb{E}[f(x_{n_{l+1}})] \nonumber \\
	 & \quad =  \mathbb{E}[f(x_{n_{0}})] + \sum_{i=0}^{l} \mathbb{E}[f(x_{n_{i+1}}) - f(x_{n_{i}})]. \label{eq:main_eq_esterr}
	\end{align}
	Now to bound $ \mathbb{E}[f(x_{n_{i+1}}) - f(x_{n_{i}})] $, we first consider the case when $i\geq 1$. Using Lemma~\ref{lm:biased_convex_sp_esterr} with $k_0 =n_i$ and $k_1 = N_{i+2}$, we obtain
	\begin{align}
	&\frac{\sum_{k=n_i}^{N_{i+2}}2\gamma_k\mathbb{E}\left[f(x_k)-f(x_{n_i})\right] }{N_{i+2}-n_i + 1}\nonumber \\
	&\leq \frac{\sum_{k=n_i}^{N_{i+2}}\bigg( 2 \sqrt{d} \gamma_k D \mathcal{E}_k + \gamma _ { k } ^ { 2 } \mathcal{B}_k^2   \bigg) }{N_{i+2}-n_i+1} \nonumber \\
	& \leq 2 \sqrt{d}  \gamma_{N_{i+1}} D \mathcal{E}_{N_{i+1}} + \mathcal{B}_{N_{i+1}}^2 \gamma_{N_{i+1}}^2 \label{eq:dec_step-size} \\
& = \frac{2 \sqrt{d} D \gamma_0 }{N} \left(  c_1 \eta_0^2  2^{-3i/2} + \frac{ c_3 2^{-5i/4}}{ \eta_0 }  \right)  + \frac{2^{-2i} \gamma_0^2}{N^{4/3}} \bigg[ \nonumber \\
& \quad B^2 + \frac{2  \sqrt{d} B c_1 \eta_0^2 2^{-i/2}}{ N^{1/3}}  + \frac{2 \sqrt{d} B c_3 2^{-i/4}}{ \eta_0 N^{1/3}} + \frac{d c_1^2 \eta_0^4 2^{-i}}{ N^{2/3}} \nonumber \\
& \quad + \frac{2 {d}  c_1 c_3 \eta_0 2^{-3i/4}}{N^{2/3}} +  \frac{d  c_3^2 2^{-i/2}}{ \eta_0^2 N^{2/3}} + \frac{c_2 N^{1/3}}{\eta_0^2 2^{-i/2}} \bigg]. \label{eq:firsk_equation_esterr} 
	\end{align}	 
	The inequality in \eqref{eq:dec_step-size} follows from the fact that $\gamma_k$ and $ \eta_k $  are decaying in a phase-dependent manner (see \eqref{eq:weak_step_size_sp_esterr_local}).
	Note that from the definition of $n_i$, $\mathbb{E}[f(x_{k}) - f(x_{n_i})] \geq 0$ whenever $N_i<k \leq N_{i+1}$. Thus, we have 
	\begin{align}
	&\frac{\sum_{k=n_i}^{N_{i+2}}2\gamma_k\mathbb{E}\left[f(x_k)-f(x_{n_i})\right] }{N_{i+2}-n_i + 1} \nonumber \\&\geq \frac{\sum_{k=N_{i+1}+1}^{N_{i+2}}2\gamma_k\mathbb{E}\left[f(x_k)-f(x_{n_i})\right] }{N_{i+2} - n_i +1 } \nonumber\\
	&\geq 2\gamma_{N_{i+2}}\frac{N_{i+2} - N_{i+1}}{N_{i+2}-N_i}\mathbb{E}\left[f(x_{n_{i+1}})-f(x_{n_i})\right] \nonumber \\& \geq \frac{2\gamma_{N_{i+2}}}{5} \mathbb{E}\left[f(x_{n_{i+1}})-f(x_{n_i})\right] \nonumber \\
	& = \frac{2^{-i} \gamma_0 }{5 N^{2/3}} \mathbb{E}\left[f(x_{n_{i+1}})-f(x_{n_i})\right],
	\label{eq:second_equation_esterr}
	\end{align} 
	where the second inequality follows from the assumption that $ \mathbb{E}[f(x_{n_{i+1}})]\geq \mathbb{E}[f(x_{n_i})]$, and the fact that $N_{i+2} - N_i \geq N_{i+2} - n_i +1$. The last inequality follows from the Lemma 4 of \cite{jain2019making}. Combining \eqref{eq:firsk_equation_esterr} and \eqref{eq:second_equation_esterr}, we obtain
	\begin{align}
	&\mathbb{E}[f(x_{n_{i+1}}) - f(x_{n_{i}})] \nonumber \\
	& \leq \frac{10 \sqrt{d} D c_1 \eta_0^2  2^{-i/2}}{ N^{1/3}} + \frac{10 \sqrt{d} D c_3 2^{-i/4}}{ \eta_0 N^{1/3}}  + \frac{5 \gamma_0 2^{-i}}{N^{2/3}}\bigg[ \nonumber \\
	& \quad B^2  + \frac{2  \sqrt{d} B c_1 \eta_0^2 2^{-i/2}}{ N^{1/3}}  + \frac{2 \sqrt{d} B c_3 2^{-i/4}}{ \eta_0 N^{1/3}} + \frac{d c_1^2 \eta_0^4 2^{-i}}{ N^{2/3}} \nonumber \\
	& \quad + \frac{2 {d}  c_1 c_3 \eta_0 2^{-3i/4}}{N^{2/3}} +  \frac{d  c_3^2 2^{-i/2}}{ \eta_0^2 N^{2/3}} + \frac{c_2 N^{1/3}}{\eta_0^2 2^{-i/2}} \bigg]. \label{eq:phases_bound_esterr}
	\end{align}
	This completes the proof for the case when $ i\geq 1 $. The proof for the case when $i=0$ follows in a similar manner. Plugging \eqref{eq:phases_bound_esterr} into \eqref{eq:main_eq_esterr}, we obtain
	\begin{align}
	&\mathbb{E}[f(x_{N})] = \mathbb{E}[f(x_{n_{l+1}})]\nonumber \\
	& =  \mathbb{E}[f(x_{n_{0}})] + \sum_{i=0}^{l} \mathbb{E}[f(x_{n_{i+1}}) - f(x_{n_{i}})] \nonumber\\
	& \leq \mathbb{E}[f(x_{n_{0}})] +  \frac{10 \sqrt{d} D c_1 \eta_0^2 }{ N^{1/3}} + \frac{10 \sqrt{d} D c_3 }{ \eta_0 N^{1/3}}  + \frac{5 \gamma_0 }{N^{2/3}}\bigg[ \nonumber \\
	& \quad B^2  + \frac{2  \sqrt{d} B c_1 \eta_0^2 }{ N^{1/3}}  + \frac{2 \sqrt{d} B c_3 }{ \eta_0 N^{1/3}} + \frac{d c_1^2 \eta_0^4 }{ N^{2/3}} + \frac{2 {d}  c_1 c_3 \eta_0 }{N^{2/3}}  \nonumber \\
	& \quad +  \frac{d  c_3^2 }{ \eta_0^2 N^{2/3}} + \frac{c_2 N^{1/3}}{\eta_0^2 } \bigg] +\sum_{i=1}^{l}  \bigg( \frac{10 \sqrt{d} D c_1 \eta_0^2  2^{-i/2}}{ N^{1/3}}  \nonumber \\
	& \quad + \frac{10 \sqrt{d} D c_3 2^{-i/4}}{ \eta_0 N^{1/3}}  + \frac{5 \gamma_0 2^{-i}}{N^{2/3}}\bigg[ B^2  \nonumber \\
	& \quad + \frac{2  \sqrt{d} B c_1 \eta_0^2 2^{-i/2}}{ N^{1/3}}  + \frac{2 \sqrt{d} B c_3 2^{-i/4}}{ \eta_0 N^{1/3}} + \frac{d c_1^2 \eta_0^4 2^{-i}}{ N^{2/3}} \nonumber \\
	& \quad + \frac{2 {d}  c_1 c_3 \eta_0 2^{-3i/4}}{N^{2/3}} +  \frac{d  c_3^2 2^{-i/2}}{ \eta_0^2 N^{2/3}} + \frac{c_2 N^{1/3}}{\eta_0^2 2^{-i/2}} \bigg]\bigg) \nonumber \\ 
	& \leq \inf_{ \lceil \frac{N}{4}\rceil \leq k \leq N_1}\mathbb{E}[f(x_k)]  + \frac{ \sqrt{d} D (35{c_1 \eta_0^2} + \frac{63c_3}{\eta_0})}{N^{1/3}} \nonumber \\
	& \quad + \frac{10 \gamma_0 B^2}{N^{2/3}}  + \frac{20  \gamma_0 \sqrt{d}  B (c_1\eta_0^2 + \frac{c_3}{\eta_0}) }{N}  \nonumber \\
	& \quad + \frac{10 \gamma_0 d ({c_1 \eta_0^2}+ \frac{c_3}{\eta_0})^2 }{N^{4/3}}+ \frac{ 17.5  \gamma_0 c_2 }{ \eta_0^2 N^{1/3}} \label{eq:main_eq_2_esterr}.
	\end{align}
	Note that for all $ k \leq N_1 $, we have step size $\gamma_k = \frac{\gamma_0}{N^{2/3}}$ and perturbation constant $ \eta_k = \frac{\eta_0}{ N^{1/6}}$. Let $x_k$ be the output of SGD-BGO algorithm, then using the fact that infimum is smaller than any weighted average, we have 
	\begin{align}
	&\inf_{\lceil\frac{N}{4}\rceil \leq k\leq N_1} \mathbb{E}[f(x_k) - f(x^{*})] \nonumber \\& \leq \frac{1}{N_1 - \lceil\frac{N}{4}\rceil +1} \sum_{k=\lceil\frac{N}{4}\rceil }^{N_1} \mathbb{E}[f(x_k) - f(x^{*})] \nonumber
	\\ & \leq  \frac{2}{N_1}\sum_{k=1 }^{N_1} \mathbb{E}[f(x_k) - f(x^{*})] \label{eq:n1_less} \\
	& \leq \frac{2}{{N_1}} \bigg[\frac{D^{2} N^{2/3}}{2 \gamma_0}   +\frac{ {B}^2 N_1 \gamma_0}{2 N^{2/3}} + \frac{2 N_1  D \sqrt{d} c_1 \eta_0^2}{N^{1/3}} \nonumber \\
	& \quad + \frac{2 N_1  D \sqrt{d} c_3}{ \eta_0 N^{1/3}}\bigg] \label{eq:lm2} \\
	& \leq \frac{4 D^{2} N^{2/3}}{\gamma_0 N}   +\frac{\gamma_0 {B}^2 }{N^{2/3}} +  \frac{4 D \sqrt{d} c_1 \eta_0^2 }{ N^{1/3}} + \frac{4 D \sqrt{d} c_3}{ \eta_0 N^{1/3}} \nonumber \\
	& = \frac{1}{N^{1/3}} \left[ \frac{4 D^2}{\gamma_0} + \frac{ \gamma_0 {B}^2}{N^{1/3}} + 4 D \sqrt{d} \left( {c_1 \eta_0^2} + \frac{c_3}{\eta_0} \right)  \right] \label{eq:thm_convex_sp_esterr},
	\end{align}
where the inequality in \eqref{eq:n1_less} follows from the fact that $N_1 \leq 2(N_1 - \lceil\frac{N}{4}\rceil +1)$, the inequality in \eqref{eq:lm2} follows from the Lemma \ref{lm:convex_sp_std_analysis_esterr} and the final inequality follows from the fact that $\frac{N}{4}\leq N_1 \leq \frac{N}{2}$.	
	We conclude by plugging \eqref{eq:thm_convex_sp_esterr} into \eqref{eq:main_eq_2_esterr} to obtain
	\begin{align*}
	& \mathbb{E}[f(x_N)] - f(x^*) \\
	& \leq  \frac{1}{N^{1/3}} \left[ \frac{4 D^2}{\gamma_0} + \frac{ \gamma_0 {B}^2}{N^{1/3}} + 4 D \sqrt{d} \left( {c_1 \eta_0^2} + \frac{c_3}{\eta_0} \right)   \right] \\
	& \quad + \frac{ \sqrt{d} D (35{c_1 \eta_0^2} + \frac{63c_3}{\eta_0})}{N^{1/3}} + \frac{10 \gamma_0 B^2}{N^{2/3}} + \frac{ 17.5  \gamma_0 c_2 }{ \eta_0^2 N^{1/3}} \nonumber \\
	& \quad + \frac{20  \gamma_0 \sqrt{d}  B (c_1\eta_0^2 + \frac{c_3}{\eta_0}) }{N}  + \frac{10 \gamma_0 d ({c_1 \eta_0^2}+ \frac{c_3}{\eta_0})^2 }{N^{4/3}} . \nonumber 
	\end{align*}
\end{proof}

\subsubsection{Proof of Theorem~\ref{thm:convex_gs_esterr}} \label{pf:biased_convex_gs_esterr}

The proof proceeds through a sequence of lemmas, similar to the proof of Theorem~\ref{thm:convex_sp_esterr} in Section \ref{pf:biased_convex_sp_esterr} under the oracle \ref{as:biased_sp_estimation_error}.

\begin{lemma} \label{lm:biased_convex_gs}
	Assume \ref{as:boundedness}. With the oracle \ref{as:biased_gs_estimation_error}, suppose that the SGD-BGO algorithm is run with stepsize sequence $ \{\gamma_{ k }\}_{k=1}^{N} $. Then, given any $ 1 < k_0 < k_1 \leq N $, we have
	\begin{align*}
		\sum _{ k = k_0 } ^ { k_1 } 2 \gamma_k \mathbb{E} \left[ f \left( x _ { k } \right) - f (x_{k_0}) \right]  \leq 
		\sum _ { k = k_0 } ^ { k_1 } 2 \sqrt{d} \gamma_k  D\mathcal{E}_k + \gamma _ { k } ^ { 2 } \mathcal{B}_k^2,
	\end{align*}
	where $ \mathcal{B}_k^2 = \bigg[ B^2  + 2\sqrt{d} B \mathcal{E}_k  + d\mathcal{E}_k ^2 + {c_2}{\eta_k^2} + \tilde{c_2}\bigg] $, $ \mathcal{E}_k = c_1 \eta_k  + \frac{c_3}{\eta_k \sqrt{m_k}}  $, constants $ c_1,c_2 $ is as defined in \ref{as:biased_gs_estimation_error}, $ B $ is as defined in \ref{as:boundedness}, and $ D $ is given in \eqref{eq:dia}.
\end{lemma}

\begin{proof}
	Follows by a completely parallel argument to the proof of Lemma \ref{lm:biased_convex_sp_esterr}, after observing that $ \mathbb { E }_{\xi_{ [k] }}  \left[ g \left( x _ { k } , \xi _ { k } \right) \right]  \leq \nabla f \left( x _ { k } \right) + c_1 \eta_k \mathbf{1}_{d \times 1} + \frac{c_3}{\eta_k \sqrt{m_k}}\mathbf{1}_{d \times 1}  ,  $ and $ \mathbb { E }_{\xi_{[k]}} \left[ \left\| g \left( x _ { k } , \xi _ { k } \right) \right\|^{2} \right] \le \left\| \mathbb { E }_{\xi_{[k]}}  \left[ g \left( x _ { k } , \xi _ { k } \right) \right]  \right\| ^{2} + c_2\eta_k^2 + \tilde{c_2} $.
	
\end{proof}

\begin{lemma} \label{lm:convex_gs_std_analysis}
	Assume \ref{as:boundedness}. With the oracle \ref{as:biased_gs_estimation_error}, suppose that the SGD-BGO algorithm is run with a constant stepsize and perturbation constant, i.e., $ \gamma_k = \gamma, \eta_k = \eta, \forall k \geq 1 $. Then, for any $ k \geq 1$, we have	
	\[
	\sum_{k=1}^{N} \E \left[f\left(x_{k}\right)-f\left(x^{*}\right)\right] \leq \frac{D^{2}}{2 \gamma}  +2 N D \sqrt{d}  \mathcal{E} + \frac{ N \gamma B^2}{2}, 
	\]	
	where $ \mathcal{E} = c_1\eta + \frac{c_3}{\eta \sqrt{m}} $, $ c_1 $ is as defined in \ref{as:biased_gs_estimation_error}, $ B $ is as defined in \ref{as:boundedness}, and $ D $ is as defined in \eqref{eq:dia}.
\end{lemma}

\begin{proof}
	The proof follows in a similar manner as that of Lemma \ref{lm:convex_sp_std_analysis_esterr}, 
	with the following modification: $ \E [ \Delta g_k ] =  c_1 \eta_k \mathbf{1}_{d \times 1} + \frac{c_3}{\eta_k \sqrt{m_k}}\mathbf{1}_{d \times 1}  $.	
\end{proof}

\begin{proof} ~\textbf{(Theorem~\ref{thm:convex_gs_esterr})}
	Using a parallel argument to the initial passage in the proof of Theorem \ref{thm:convex_sp_esterr} leading upto equation \eqref{eq:phases_bound_esterr}, we obtain
	\begin{align}
		&\mathbb{E}[f(x_{n_{i+1}}) - f(x_{n_{i}})] \leq \frac{10 \sqrt{d} D c_1 \eta_0  2^{-i}}{N} \nonumber \\
		&  + \frac{10 \sqrt{d}  c_3 D 2^{-i/2}}{ \eta_0 \sqrt{N}}  + \frac{5 \gamma_0 2^{-i}}{\sqrt{N}}\bigg[B^2  + \frac{2  \sqrt{d} B c_1 \eta_0 2^{-i}}{N}  \nonumber \\
		& + \frac{2 \sqrt{d} B c_3  2^{-i/2}}{ \eta_0 \sqrt{N}}  + \frac{(dc_1^2  + c_2) \eta_0^2 2^{-2i}}{N^{2}} \nonumber \\
		& + \frac{2 c_1 c_3 d 2^{-3i/2}}{N^{3/2}} +  \frac{d c_3^2 2^{-i}}{ \eta_0^2 N}  + \tilde{c_2}  \bigg].  \label{eq:phases_bound_gs}
	\end{align}
	Plugging \eqref{eq:phases_bound_gs} into \eqref{eq:main_eq_esterr}, we get
	\begin{align}
		& \mathbb{E}[f(x_{N})] = \mathbb{E}[f(x_{n_{l+1}})] \leq \inf_{ \lceil \frac{N}{4}\rceil \leq k \leq N_1}\mathbb{E}[f(x_k)] \nonumber \\
		& + \frac{20 \sqrt{d} D c_1 \eta_0}{N} + \frac{35 \sqrt{d}  c_3 D }{ \eta_0 \sqrt{N}} + \frac{10 \gamma_0 B^2}{\sqrt{N}} \nonumber \\
		& + \frac{20 \sqrt{d} B  c_1 \gamma_0  \eta_0 }{ N^{3/2}} + \frac{20 \sqrt{d}  B c_3 \gamma_0 }{ \eta_0 N}    + \frac{10 \gamma_0 \eta_0^2 (dc_1^2 + c_2) }{N^{5/2}} \nonumber \\
		& + \frac{20 \gamma_0 c_1 c_3 {d} }{N^{2}} +  \frac{10  d \gamma_0  c_3^2 }{ \eta_0^2 N^{3/2}} + \frac{ 10 \gamma_0 \tilde{c_2} }{\sqrt{N}}. \label{eq:main_eq_2_gs}
	\end{align}
	As in the proof of Theorem \ref{thm:convex_sp_esterr}, we obtain
	\begin{align}
		&\inf_{\lceil\frac{N}{4}\rceil \leq k\leq N_1} \mathbb{E}[f(x_k) - f(x^{*})] \nonumber\\
			& \leq  \frac{2}{N_1}\sum_{k=1 }^{N_1} \mathbb{E}[f(x_k) - f(x^{*})] \nonumber \\
		& \leq  \frac{1}{\sqrt{N}} \left[ \frac{4D^2 }{\gamma_0} + {\gamma_0 {B}^2} + \frac{4 D \sqrt{d}c_1 \eta_0 }{\sqrt{N}} + \frac{4D \sqrt{d} c_3}{\eta_0}\right], \label{eq:thm_convex_gs}
	\end{align}
	where we used Lemma \ref{lm:convex_gs_std_analysis} and the fact that $\frac{N}{4}\leq N_1 \leq \frac{N}{2}$.	
	We conclude by plugging \eqref{eq:thm_convex_gs} into \eqref{eq:main_eq_2_gs} to obtain
	\begin{align*}
		& \mathbb{E}[f(x_N)] - f(x^*) 
		\leq \frac{1}{\sqrt{N}} \bigg[ \frac{4D^2}{\gamma_0} + {11 \gamma_0 {B}^2} \\
		&+ {39 D \sqrt{d} } \left( \frac{c_1 \eta_0}{\sqrt{N}}  + \frac{ c_3}{\eta_0}  \right) +  \frac{20 \sqrt{d} B \gamma_0}{\sqrt{N}}  \left(  \frac{c_1 \eta_0}{\sqrt{N}}  + \frac{ c_3}{\eta_0}  \right)  \\
		& + \frac{10 d \gamma_0}{N} \left( \frac{c_1 \eta_0}{\sqrt{N}}  + \frac{ c_3}{\eta_0}  \right)^2  + \frac{10 \gamma_0 \eta_0^2 c_2}{N^2} + {10 \gamma_0 \tilde{c_2}}\bigg].
	\end{align*}
\end{proof}

\subsection{Proof for Risk-Sensitive Reinforcement Learning} \label{pf:biased_convex_sp}

\subsubsection{Proof of Proposition \ref{prop:risk_measure}} \label{pf:prop_risk_measure}

\begin{proof}
	Let $F$ denote the distribution of $K_x(x^0)$. Then, we have
	\begin{align*}
		\E\left| \rho_m - \rho(K_x(x^0))\right| &=  \E\left| \rho(F_m) - \rho(F)\right| \\
		&\le L W_1(F,G) \le  \frac{c_1 L B}{\sqrt{m}}, 
	\end{align*}
	where the final inequality follows by using Theorem 3.1 of \cite{lei2020convergence}.
\end{proof}

\section{Conclusions}
\label{sec:concl}

Motivated by practical applications involving biased function measurements, 
we formulated two biased gradient oracles with an additive estimation error component. The first oracle featured a bias-variance tradeoff for the gradient estimates, while the second one did not have such a tradeoff. 
We studied the non-asymptotic performance of gradient-based algorithms with inputs from a biased gradient oracle in convex as well as non-convex optimization regimes. 
Further, we highlighted the applicability of our proposed biased gradient oracles in a risk-sensitive reinforcement learning setting.

\bibliographystyle{IEEEtran}
\bibliography{references}

\end{document}